%% file: main.tex
\documentclass{article} 
\usepackage{times}
\usepackage{fullpage}

\usepackage{natbib}

\input{math_commands.tex}

\usepackage{hyperref}
\usepackage{url}

\usepackage{amsthm}
\usepackage{amsbsy}
\usepackage{mathtools}
\usepackage{mdframed} 
\usepackage{thmtools}

\newcommand{\alg}[1]{$\mathrm{#1}$}
\definecolor{shadecolor}{gray}{0.90}
\declaretheoremstyle[
headfont=\normalfont\bfseries,
notefont=\mdseries, notebraces={(}{)},
bodyfont=\normalfont,
postheadspace=0.5em,
spaceabove=6pt,
mdframed={
  skipabove=8pt,
  skipbelow=8pt,
  hidealllines=true,
  backgroundcolor={shadecolor},
  innerleftmargin=4pt,
  innerrightmargin=4pt}
]{shaded}

\declaretheorem[style=shaded,within=section]{definition}

\declaretheorem[style=shaded,sibling=definition]{proposition}

\declaretheorem[style=shaded,sibling=definition]{lemma}

\usepackage[colorinlistoftodos,bordercolor=orange,backgroundcolor=orange!20,linecolor=orange]{todonotes}

\usepackage{algorithm}
\newcommand{\Eqmark}[2]{\stackrel{(#1)}{#2}}

\usepackage{booktabs}
\usepackage{caption}
\usepackage{subcaption}

\usepackage{algpseudocode}
\usepackage{tikz}
\usetikzlibrary{fit,calc,tikzmark}
\makeatletter
\algrenewcommand\algorithmicrequire{\textbf{Input:}}
\algrenewcommand\alglinenumber[1]{\textcolor{gray!70}{\footnotesize#1}}
\algrenewcommand\algorithmiccomment[1]{\hfill\textcolor{gray!70}{\scriptsize$\triangleright$~#1}}
\makeatother
\newcommand{\AlgIO}[2]{\noindent\textbf{#1:}~#2\par}
 
\definecolor{algC1}{HTML}{1F77B4} 
\definecolor{algC2}{HTML}{2CA02C} 
\definecolor{algC3}{HTML}{FA8072} 

\newlength{\AlgBoxPadLeft}   
\newlength{\AlgBoxPadRight}  
\newlength{\AlgBoxPadTop}    
\newlength{\AlgBoxPadBottom} 
\newcommand{\AlgBox}[4]{%
\begin{tikzpicture}[remember picture,overlay]
\node[
draw=#3, rounded corners, ultra thick, inner sep=6pt,
fit={( $ (pic cs:#1) + (-\AlgBoxPadLeft, \AlgBoxPadTop) $ ) ( $ (pic cs:#2) + ( \AlgBoxPadRight, -\AlgBoxPadBottom) $ )}
] (box-#1) {};
\node[
anchor=south west, yshift=2pt, xshift=4pt,
fill=#3!10, draw=#3, rounded corners=2pt, inner xsep=4pt, inner ysep=1pt,
font=\scriptsize\bfseries
] at (box-#1.north west) {#4};
\end{tikzpicture}%
}

\usepackage{xcolor}
\usepackage{listings}

\usepackage[scaled=0.92]{inconsolata}

\definecolor{codebg}{RGB}{250,250,250}
\definecolor{codeframe}{RGB}{215,215,215}
\definecolor{codekw}{RGB}{0,69,134}
\definecolor{codestr}{RGB}{160,82,45}
\definecolor{codecm}{RGB}{0,102,102}
\definecolor{codenum}{gray}{0.45}

\lstdefinestyle{mypython}{
  language=Python,
  basicstyle=\ttfamily\footnotesize,
  backgroundcolor=\color{codebg},
  keywordstyle=\color{codekw}\bfseries,
  stringstyle=\color{codestr},
  commentstyle=\color{codecm}\itshape,
  numberstyle=\scriptsize\color{codenum},
  showstringspaces=false,
  keepspaces=true,
  columns=fullflexible,
  frame=single,
  frameround=ffff,
  rulecolor=\color{codeframe},
  framesep=5pt,
  xleftmargin=0pt,
  xrightmargin=0pt,
  aboveskip=4pt,
  belowskip=0pt,
  tabsize=2
}

\newcommand{\psqueeze}{\vspace{-0.0cm}}
\usepackage{titlesec}
\definecolor{sectionblue}{RGB}{0,60,125}
\titleformat{\section}{\normalfont\Large\bfseries\color{sectionblue}}{\thesection}{1em}{}
\titleformat{\subsection}{\normalfont\large\bfseries\color{sectionblue}}{\thesubsection}{1em}{}

\usepackage{titling}
\pretitle{\begin{center}\begingroup\huge\bfseries\color{sectionblue}}
\posttitle{\par\endgroup\end{center}}

\makeatletter
  \renewcommand*{\@fnsymbol}[1]{%
    \ensuremath{%
      \ifcase#1\or 
      1\or 
      2\or 
      3\or 
      \else\@ctrerr\fi%
    }%
  }
\makeatother

\title{An Exploration of Non-Euclidean Gradient Descent: {\alg{Muon}} and its Many Variants}

\author{Michael Crawshaw\thanks{Department of Computer Science, George Mason University, \texttt{\{mcrawsha,mingruil\}@gmu.edu}}, $\quad$
Chirag Modi\thanks{Center for Cosmology and Particle Physics, New York University, \texttt{modichirag@nyu.edu}}, $\quad$
Mingrui Liu$^1$, $\quad$
Robert M. Gower\thanks{Center for Computational Mathematics, Flatiron Institute, \texttt{rgower@flatironinstitute.org}}
}

\begin{document}

\maketitle

\begin{abstract}
To define a steepest descent method over a neural network, we need to choose a norm for each layer, a way to aggregate these norms across layers, and whether to use normalization.
We systematically explore different alternatives for aggregating norms across layers, both formalizing existing combinations of {\alg{Adam}} and the recently proposed {\alg{Muon}} as a type of non-Euclidean gradient descent, and deriving new variants of the {\alg{Muon}} optimizer.
Through a comprehensive experimental evaluation of the optimizers within our framework, we find that {\alg{Muon}} is sensitive to the choice of learning rate, whereas a new variant we call {\alg{MuonMax}} is significantly more robust.
We then show how to combine any Non-Euclidean gradient method with model based momentum (known as \alg{Momo}).
The new {\alg{Momo}} variants of {\alg{Muon}} are significantly more robust to hyperparameter tuning, and often achieve a better validation score.
Thus for new tasks, where the optimal hyperparameters are not known, we advocate for using {\alg{Momo}} in combination with {\alg{MuonMax}} to save on costly hyperparameter tuning.
\end{abstract}

\begin{figure}[h]
\centering
\includegraphics[width=0.49\textwidth, height=0.33\textwidth]{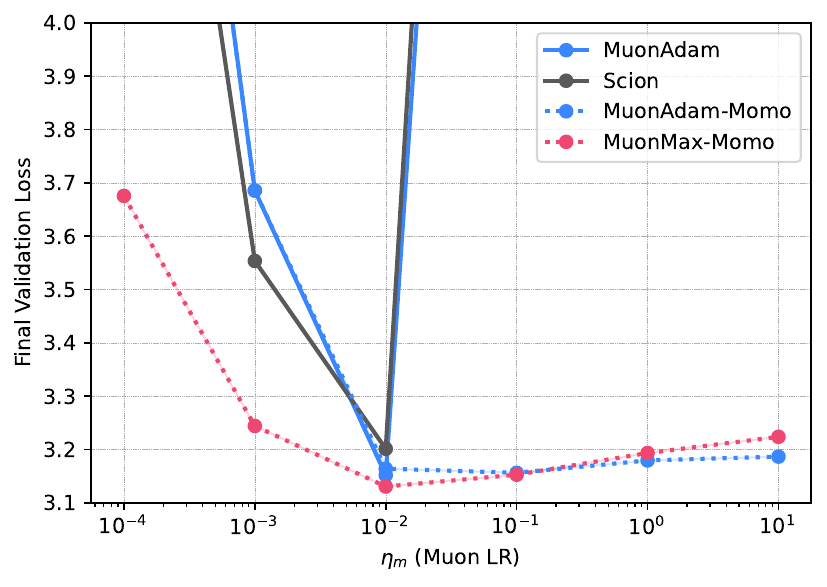}
\includegraphics[width=0.49\textwidth]{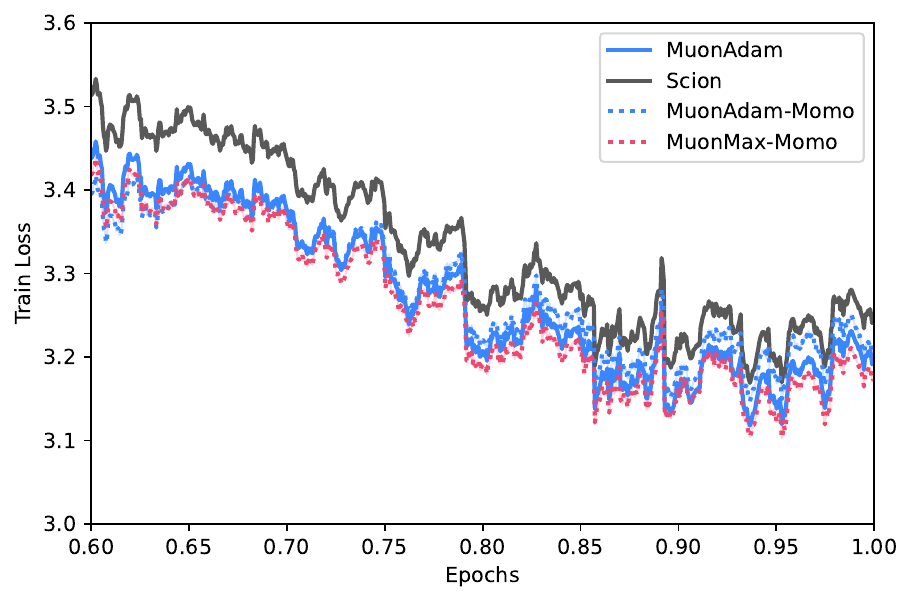}
\caption{Learning rate sweep for training GPT2-Large (774M params) on SlimPajama with 1B
tokens. \textbf{Left:} Final validation loss for various learning rates. \alg{MuonAdam}
and \alg{Scion} require precise tuning, whereas our \alg{MuonAdam}-\alg{Momo} and
\alg{MuonMax}-\alg{Momo} achieve low loss for a significantly wider range of learning
rates. \textbf{Right:} Training loss (with tuned LRs) for the last 40\% of steps.}
\label{fig:slim_pajama1B}
\end{figure}

\section{Introduction}
The recently proposed \alg{Muon} optimizer \citep{muon} has generated increasing
interest due to its efficiency for training language models \citep{scion,moonlight}.
This algorithm was initially introduced \citep{anthology,muon} and often interpreted
\citep{scion, kovalev2025understanding, fan2025implicit} as steepest descent with
respect to the spectral norm for each weight matrix in a neural network.

However, this interpretation does not entirely apply for practical implementations of
\alg{Muon}. In practice, \alg{Muon} is used side-by-side with another optimizer, where
hidden weight matrices are trained with \alg{Muon}, and all other parameters by
\alg{Adam} \citep{muon,moonlight,modded_nanogpt_2024} or \alg{SignDescent}
\citep{scion}. We will refer to this combination as \alg{MuonAdam} throughout, see
Algorithm~\ref{alg:muonadam} in  the appendix. Furthermore, for the weight matrices only
the normalized version of \alg{Muon} has been explored in practice.


Here we aim to strengthen the theoretical foundation of \alg{MuonAdam} and develop new
optimizers by systematically investigating different design choices which have not been
explored.
We introduce a framework for steepest descent on the entire space of network parameters,
which involves a choice of norm for each layer, a \emph{product norm} to aggregate norms
across layers, and whether to normalize updates. This framework encompasses
\alg{MuonAdam} and other variations, which provides a more principled interpretation of
these algorithms as genuine steepest descent on the entire space of network parameters,
and also opens a design space for previously unexplored \alg{Muon}-type algorithms.

One such unexplored variant is what we call \alg{MuonMax}, that arises from a new
product norm and does not use update normalization. The updates of \alg{MuonMax} depend
on the nuclear norm of the momentum from every weight matrix, which is slightly less
efficient per-step than \alg{Muon}. To make \alg{MuonMax} more efficient, we introduce a
stale approximation of these nuclear norms, which can be implemented with near-identical
memory and $5\%$ additional wall-clock time per step as \alg{Muon}.

Now that we can frame \alg{MuonAdam} and other variants as a type of steepest descent, we can import other tools used for steepest descent gradient methods.
One such tool is \alg{Momo} \citep{momo}, an adaptive step size based on model
truncation~\citep{AsiDuchi2019PNAS} that increases robustness to learning rate tuning.
We extend the \alg{Momo} step size to general steepest descent for arbitrary norms and
subsequently apply it to the algorithms in our framework.

We perform a systematic evaluation of many algorithms in our framework for training GPT
models with up to 774M parameters for language modeling on the FineWeb
\citep{penedo2024fineweb} and SlimPajama \citep{cerebras2023slimpajama} datasets with up
to 6B tokens. We find that \alg{MuonMax}-\alg{Momo} consistently matches or outperforms
\alg{MuonAdam} and \alg{Scion}~\citep{scion} while enjoying a much larger range of
competitive learning rates, meaning that \alg{MuonMax}-\alg{Momo} is much less sensitive
to tuning. We also observe that \alg{Momo} increases tuning robustness for all
variations and that our stale nuclear norm approximation causes negligible change in
performance, while decreasing wall-clock time per iteration. Our contributions are:

\begin{enumerate}
    \item \textbf{Formalizing \alg{MuonAdam}.} We introduce a steepest descent framework
        that encompasses the practical implementation of \alg{Muon} (with \alg{Adam}
        used for a subset of parameters), 
        demonstrating that even these side-by-side optimizers can be interpreted as
        steepest descent with respect to certain norms on the space of \textit{all}
        network parameters. This solidifies the theoretical foundation for practical
        variants of \alg{Muon}, and sheds light on unexplored aspects of \alg{Muon}'s
        design. Our framework also includes \alg{Scion} and other existing variants.
    \item \textbf{Defining non-Euclidean \alg{Momo}.} We show how to incorporate the
        adaptive step size \alg{Momo} with any steepest descent algorithm, which we find
        significantly increases robustness to the learning rate tuning.
    \item \textbf{\alg{MuonMax}: New practical, robust variant of \alg{Muon}.} We
        propose a new optimizer, \alg{MuonMax}, which arises within our framework from a
        novel product norm. With a stale approximation of the nuclear norm of each
        layer's momentum, \alg{MuonMax} has near-identical memory cost and $5\%$
        additional wall-clock time per iteration compared to \alg{Muon}.
    \item \textbf{Systematic Evaluation.} We perform a comprehensive evaluation of
        optimizers in our framework for language modeling. \alg{MuonMax}-\alg{Momo}
        consistently matches or outperforms \alg{Muon} and other baselines while
        widening the range of competitive step size choices by several orders of
        magnitude.
\end{enumerate}

\paragraph{Notation.}
We use $\langle \cdot, \cdot \rangle$ to denote the Euclidean inner product on $\R^d$ or
on products of the form $\R^{d_1} \times \ldots \times \R^{d_n}$ naturally by viewing
elements of $\R^{d_1} \times \ldots \times \R^{d_n}$ as elements of $\R^{d_1 + \ldots +
d_n}$. Note that for matrices, which can also be viewed as elements of $\R^{mn}$, this
inner product is consistent with the trace inner product, since $\text{Tr}(\mA^T \mB) =
\langle \text{vec}(\mA), \text{vec}(\mB) \rangle$.


\section{Related Work}

\paragraph{Muon.} 
The use of \emph{Spectral descent}, that is steepest descent with respect to the
spectral norm, on deep neural networks dates back to~\citet{carlson2015stochastic,
carlson2015preconditioned}. \alg{Muon} is the combination of spectral descent with
momentum~\citep{anthology}, and a carefully crafted polynomial algorithm for computing
the polar factor~\citep{muon}. Recent work has designed an optimal such polynomial
algorithm for the polar factor called \alg{PolarExpress}~\citep{amsel2025polar}, which
we use in our \alg{Muon} implementation. \citet{scion} introduced \alg{Scion}, which
uses \alg{SignSGD} with momentum (instead of \alg{Adam}) to train non-matrix parameters.
\citet{moonlight} scaled \alg{Muon} to train a 16B parameter language model with 5.7T
tokens. Several works have developed theory of \alg{Muon}'s convergence
\citep{li2025note, kovalev2025understanding, riabinin2025gluon} and implicit bias
\citep{tsilivis2025flavors, fan2025implicit}.

Most similar to ours is the line of work developing the modular norm \citep{anthology,
large2024scalable, modular_duality}. This line of work also argues that we should
perform steepest descent on the entire space of network parameters, instead of
separately at each layer, and focuses on steepest descent with respect to a particular
norm called the modular norm. This norm enables Lipschitz continuity of the neural
network with respect to both weights and inputs. In this work, we take an orthogonal
approach, where we develop a general theory of steepest descent on product spaces, and
numerically investigate many possible norms on these spaces. We are not aware of any
existing evaluation of steepest descent with respect to the modular norm. \psqueeze

\paragraph{Model Truncation.} 
Gradient descent can be viewed as using the local linearization of the loss as a \emph{model} of the loss itself. If we know a lower bound of the loss, for instance most loss functions are positive, then we can improve this linear model by \emph{truncating} the model at this lower bound~\citep{Asi2019}. Follow-up work emphasizes the importance of such model choices in stochastic optimization~\citep{AsiDuchi2019PNAS}, and extends the framework to minibatch settings~\citep{AsiEtAl2020NeurIPS}.
Using model truncation often leads to methods that are more stable and easier to tune~\citep{SPS,Davis2019,meng2022a,Schaipp2023}. Recently~\citet{momo} showed how to combine momentum with model truncation. Furthermore~\citet{Chen2022} combine parameter-free coin betting methods with truncated models.

\section{Steepest Descent on Neural Networks} \label{sec:steepest}
Let $F: \R^d \rightarrow \R$ be the loss function, and  
$\|\cdot\|$ be any norm on $\R^d$. We define the \emph{Linear Minimization Oracle} (LMO)
and the \emph{dual norm} as
\begin{equation}
    {\sf LMO}_{\|\cdot\|}(\vv) = \argmin_{\|\vu\|=1} \langle \vu, \vv \rangle, \quad \mbox{and} \quad \|\vv\|_* = \max_{\|\vu\|=1} \langle \vu, \vv \rangle,
\end{equation}
respectively.
When the norm is clear from context, we will omit the subscript and write ${\sf LMO}$.
Throughout we denote the stochastic gradient at step $t$ by $\vg_t$, and the momentum
buffer $\vm_t$ which is an exponential moving average of stochastic gradients, i.e.
$\vm_t = \beta \vm_{t-1} + (1 - \beta) \vg_t$ for $\beta \in [0, 1)$.

\subsection{Constrained vs Regularized Steepest Descent}
For a single weight matrix, the {\alg{Muon}}  update is often motivated as the {\sf LMO}~\citep{scion} with
respect to the spectral norm.
The following proposition shows that for a general norm, updating in the direction of
${\sf LMO}(\vm_t)$ is equivalent to minimizing a first-order Taylor approximation of $F$
around $\vw_t$, with a constraint on the update's norm and approximating $\nabla F(\vw_t) \approx \vm_t$.

\begin{restatable}{proposition}{propcsd}[Constrained Steepest Descent]\label{prop:csd}
The CSD update is given by
\begin{align} \label{eq:csd_opt}
    \vw_{t+1} &= \argmin_{\|\vw - \vw_t\| \leq \eta} \left\{ F(\vw_t) + \langle \vm_t, \vw - \vw_t \rangle \right\} \;=\; \vw_t + \eta \,{\sf LMO}(\vm_t). 
\end{align}
\end{restatable}

The ball constraint above ensures that we only use the Taylor approximation close to its
center $\vw_t$,
but another natural choice is to use regularization instead of a constraint as follows.

\begin{restatable}{proposition}{proprsd}[Regularized Steepest Descent]\label{prop:rsd}
The RSD update is given by
\begin{align} \label{eq:rsd_opt}
    \vw_{t+1} &= \argmin_{\vw} \left\{ F(\vw_t) + \langle \vm_t, \vw - \vw_t \rangle + \tfrac{1}{2 \eta} \left\| \vw - \vw_t \right\|^2 \right\} \; = \vw_t + \eta \|\vm_t\|_* {\sf LMO}(\vm_t)
\end{align}
\end{restatable}

In the case without momentum (i.e. $\beta = 0$), both of these algorithms have appeared
throughout the literature under the name steepest descent, but the recent line of work
around {\alg{Muon}}  \citep{muon,modular_duality,scion,moonlight} has mostly focused
on
the constrained variant. To the best of our knowledge, the only works which consider the
regularized variant over the space of all parameters was \citet{anthology}.
\citet{polargrad} also use the regularized interpretation of \alg{Muon} on a per layer
basis instead of the entire product space.

Notice that CSD and RSD have the same update direction, but with regularization the
update magnitude is multiplied by the dual norm of the momentum. Therefore, the primal
norm of the update $\|\vw_{t+1} - \vw_t\|$ is $\eta$ for CSD and $\eta \|\vm_t\|_*$ for
RSD. Intuitively, CSD enforces a \textit{normalized update}. \psqueeze

\subsection{Product Norms}
To describe steepest descent, we first need a norm over the space of \emph{all} network
parameters~\citep{anthology}.
Instead of flattening all parameters into a single vector, we consider the Cartesian
product $\mW = (\vw^1, \vw^2, \ldots, \vw^n)$ of network parameters (where each $\vw^i$
could be a flattened weight matrix, a bias vector, etc). We assign a norm $\|\cdot\|_{(i)}$
for parameter $\vw^i$, then aggregate these norms into a single norm on the product
space. Two natural examples of product norms are 
\begin{equation} \label{eq:product_norms}
   \textstyle \|\mW\|_{\infty} := \max_{1 \leq i \leq n} \|\vw^i\|_{(i)}, \quad \mbox{and} \quad \|\mW\|_2 := \sqrt{\sum_{i=1}^n \|\vw^i\|_{(i)}^2}.
\end{equation}
Computing the steepest descent direction with respect to a product norm requires: the
linear minimization oracle ({\sf LMO}) and the dual norm of the product norm. As we show
next, both can be expressed in terms of the underlying per-parameter norms and the norm
used to aggregate them.



\begin{restatable}{lemma}{lemproductnorm}[{\sf LMO} and Dual of Product Norms]\label{lem:product_norm}
For each $i \in [n]$, let $g_i$ be a norm on $\R^{d_i}$, and let $f$ be a norm on
$\R^n$, and denote their dual norms as $g_{i,*}$ and $f_*$, respectively. Then the
product norm $h: \R^{d_1} \times \ldots \times \R^{d_n} \rightarrow \R$ defined by
\begin{equation}
    h(\vw^1, \ldots, \vw^n) = f\big(g_1(\vw^1), \ldots, g_n(\vw^n)\big)
\end{equation}
is indeed a norm, and its ${\sf LMO}$ and dual norm are given by
\begin{align}
    {\sf LMO}_h(\vw^1, \ldots, \vw^n) &= (\phi_1 {\sf LMO}_{g_1}(\vw^1), \ldots, \phi_n {\sf LMO}_{g_n}(\vw^n)) \\
    h_*(\vw^1, \ldots, \vw^n) &= f_*(g_{1,*}(\vw^1), \ldots, g_{n,*}(\vw^n)),
\end{align}
where $(\phi_1, \ldots, \phi_n) := -{\sf LMO}_f(g_{1,*}(\vw^1), \ldots, g_{n,*}(\vw^n))$.
\end{restatable} 

We can now compute steepest descent updates (both constrained and regularized) with
respect to the product norms $\|\cdot\|_{\infty}$, $\|\cdot\|_2$, or any other product
norm  by plugging the {\sf LMO} and dual of each product norm into the steepest
descent definitions (\Eqref{eq:csd_opt} and \Eqref{eq:rsd_opt}).

Denoting by $\vm_t^i$ the momentum buffer of parameter $i$, the updates for each
parameter $\vw^i$ are:
\begin{align}
    \text{CSD w.r.t. $\|\cdot\|_{\infty}$:} \quad \vw_{t+1}^i &= \vw_t^i + \eta \, {\sf LMO}_{\|\cdot\|_{(i)}}(\vm_t^i) \label{eq:csd_infty} \\
    \text{RSD w.r.t. $\|\cdot\|_{\infty}$:} \quad \vw_{t+1}^i &= \vw_t^i + \eta \Big( \sum_{j=1}^n \|\vm_t^j\|_{(j),*} \Big) {\sf LMO}_{\|\cdot\|_{(i)}}(\vm_t^i) \label{eq:rsd_infty} \\
    \text{CSD w.r.t. $\|\cdot\|_2$:} \quad \vw_{t+1}^i &= \vw_t^i + \eta \tfrac{\|\vm_t^i\|_{(i),*}}{\sqrt{\sum_{j=1}^n \|\vm_t^j\|_{(j),*}^2}} {\sf LMO}_{\|\cdot\|_{(i)}}(\vm_t^i) \label{eq:csd_2} \\
    \text{RSD w.r.t. $\|\cdot\|_2$:} \quad \vw_{t+1}^i &= \vw_t^i + \eta \,\|\vm_t^i\|_{(i),*} {\sf LMO}_{\|\cdot\|_{(i)}}(\vm_t^i) \label{eq:rsd_2}
\end{align}

For the methods above, the update direction for each parameter $\vw_t^i$ is always the
${\sf LMO}$ of $\vm_t^i$, regardless of the choice of product norm. However, the
magnitude of each parameter's update is determined by the product norm and the dual
norms of each parameter's momentum. Therefore, different choices of the product norm
amount to different \textit{parameter-wise learning rates}.

\subsection{Incorporating Adam}

Now we show how to represent the hybrid \alg{MuonAdam} method as a steepest descent
method. For parameters $\boldsymbol{\theta}$, the \alg{Adam} update, where all vector operations are element-wise\footnote{We omit the bias correction since this bias can be removed by correctly initializing the momentum buffers~\cite{momo}. In any case it has little effect on  performance~\citep{secret_sauce}.}, is given by
\begin{equation}
    \boldsymbol{\theta}_{t+1} = \boldsymbol{\theta}_t - \eta \tfrac{\vm_t}{\sqrt{\vv_t} + \epsilon}, \quad \mbox{and}\quad \vv_{t+1} = \beta_2 \vv_t + (1 - \beta_2) \vg_t^2
\end{equation}
 \alg{Adam}  can be interpreted as steepest descent in two
different ways.

\begin{restatable}{proposition}{propadaminfty}\label{prop:adam_infty}
The $t$-th update of \alg{Adam}  is the CSD  with step size $\eta$ with respect to the norm:
\begin{equation} \label{eq:adam_infty}
    \|\boldsymbol{\theta}\|_{\text{ada}\infty} := \big\| \text{Diag} \big( \tfrac{\sqrt{\vv_t} + \epsilon}{|\vm_t|} \big) \boldsymbol{\theta} \big\|_{\infty}
\end{equation}
\end{restatable}

\begin{restatable}{proposition}{propadamtwo}\label{prop:adam_2}
The $t$-th update of \alg{Adam} is the RSD  with step size $\eta$ with respect to the norm:
\begin{equation} \label{eq:adam_2}
    \|\boldsymbol{\theta}\|_{\text{ada}2} :=  \sqrt{\big\langle \text{Diag}(\sqrt{\vv_t} + \epsilon) \boldsymbol{\theta}, \boldsymbol{\theta} \big\rangle} = \big\| \text{Diag} \big( \sqrt{\sqrt{\vv_t} + \epsilon} \big) \boldsymbol{\theta} \big\|_2
\end{equation}
\end{restatable}
Thus \alg{Adam} can be interpreted as either an adaptive trust-region sign descent
\citep{dissecting_adam, secret_sauce} or preconditioned gradient descent \citep{momo}. A
distinctive feature of these forms of steepest descent is that the norm changes over
iterations.

\subsection{The Whole Framework}
For a given neural network, we partition the parameters as $\mW = (\mW^1, \ldots, \mW^L,
\boldsymbol{\theta})$, where $\mW^1, \ldots, \mW^L$ are the hidden weight matrices and
$\boldsymbol{\theta}$ contains the  remaining parameters flattened into a single vector.
{\alg{MuonAdam}}  applies ${\sf LMO}$ updates w.r.t. the
spectral norm for the hidden weight matrices, and uses \alg{Adam} for the remainder of
the parameters, with two separate learning rates for these side-by-side optimizers,
shown in Algorithm \ref{alg:muonadam} (Appendix \ref{app:steepest_proofs}).

~\\~\\ 
\begin{restatable}{proposition}{propmuon}\label{prop:muon}
 \alg{MuonAdam} (Algorithm \ref{alg:muonadam}) is exactly CSD with step
size $\eta_m$ with respect to
\begin{equation}
    \|\mW\|_{\text{muon}} = \max \left( \max_{\ell \in [L]} \|\mW^{\ell}\|_{2 \rightarrow 2}, \tfrac{\eta_m}{\eta_b} \|\boldsymbol{\theta}\|_{\text{ada}\infty} \right).
\end{equation}
\end{restatable}

The coefficient $\eta_m/\eta_b$ effectively allows for the use of different 
learning rates for hidden weight matrices compared to all other parameters; this
is a crucial feature of {\alg{Muon}}'s speedrun implementation \citep{muon} and
of other variations \citep{scion, moonlight}.

Proposition \ref{prop:muon} shows the precise sense in which {\alg{MuonAdam}}  is a
steepest descent algorithm: it is constrained steepest descent with respect to a
particular product norm that aggregates the spectral norm of each hidden weight matrix
and an adaptive $\ell_{\infty}$ norm for all other parameters. This still leaves several
other valid choices within 
our general steepest descent framework to explore: whether to use constrained or
regularized steepest descent, which product norm to use ($\|\cdot\|_{\infty},
\|\cdot\|_2$), and which norm to assign to the non-matrix parameters
($\|\cdot\|_{\text{ada}\infty}, \|\cdot\|_{\text{ada}2}, \|\cdot\|_{\infty}$).

These three factors yield a design space for {\alg{Muon}}-type optimization algorithms,
all of which are founded on the principle of steepest descent, and most of which are
unexplored. Among these algorithms are several existing variations of {\alg{Muon}},
including \alg{Scion} \citep{scion} and \alg{PolarGrad} \citep{polargrad} (see Appendix
\ref{app:recover} for the full statements).
\psqueeze

\paragraph{Stale dual norms.}
Many of the updates we have presented so far require calculating dual norms of the
momentum buffers (e.g. \Eqref{eq:rsd_infty} through \Eqref{eq:rsd_2}). If that norm is
the spectral norm, this amounts to computing the nuclear norm of the momentum, which may
appear costly, but actually the dual norm is easy to compute once we have computed the
${\sf LMO}$, since $\|\vv\|_* = \langle -{\sf LMO}(\vv), \vv \rangle$. However, in the
case that updates are not separable across parameters, computing the dual norms of each
momentum buffer in this way requires either additional memory (to store the layer-wise
${\sf LMO}$s) or additional time (to compute the ${\sf LMO}$s twice). To see why,
consider the example of RSD with the $\|\cdot\|_{\infty}$ product norm
(\Eqref{eq:rsd_infty}), and assume for simplicity that all parameters are assigned the
spectral norm. For each layer $i$, the update $\mW_{t+1}^i = \mW_t^i - \eta \left(
\sum_{j=1}^L \|\mM_t^j\|_{\text{nuc}} \right) \text{polar}(\mM_t^i)$ cannot be executed
until $\|\mM_t^j\|_{\text{nuc}} = \langle \text{polar}(\mM_t^j), \mM_t^j \rangle$ has
been computed for every layer $j$. Crucially, the polar factors are used twice here:
once to compute dual norms, and again to update weights. So, we can either store the
polar factors for reuse (extra memory), or compute them twice (extra time); these
options are sketched in the first two columns below.

\begin{minipage}[t]{0.30\textwidth}
\centering \textbf{Extra Memory}
\begin{lstlisting}[style=mypython]
d = 0
lmos = {}
for i in range(1, L+1):
    lmos[i] = -polar(M[i])
    d -= lmos[i].dot(M[i])

for i in range(1, L+1):
    W[i] += eta * d * lmos[i]
\end{lstlisting}
\end{minipage}
\hfill
\begin{minipage}[t]{0.26\textwidth}
\centering \textbf{Extra Time}
\begin{lstlisting}[style=mypython]
d = 0
for i in range(1, L+1):
    lmo = -polar(M[i])
    d -= lmo.dot(M[i])

for i in range(1, L+1):
    lmo = polar(M[i])
    W[i] += eta * d * lmo 
\end{lstlisting} 
\end{minipage}
\hfill
\begin{minipage}[t]{0.30\textwidth}
\centering \textbf{Stale Norms}
\begin{lstlisting}[style=mypython]
new_d = 0
for i in range(1, L+1):
    lmo = -polar(M[i])
    new_d -= lmo.dot(M[i])
    W[i] += eta * old_d * lmo
old_d = new_d
\end{lstlisting} 
\hspace{-0.2cm}
\end{minipage}

    
    

The first option requires additional memory proportional to the size of the network,
while the second option doubles the wall-clock time needed to compute polar factors. As
an efficient approximation, we propose to reuse momentum dual norms from the previous
step (shown in the third column), which can be implemented without storing or
recomputing polar factors, and only requires a single additional scalar of memory for
each layer. We found in our experiments that using these ``stale" dual norms had near
negligible effect on performance. Informally, we expect this approximation to work on
the grounds that the momentum doesn't change too drastically in a single step, since
\begin{equation}
    \vm_t - \vm_{t-1} = \beta \vm_{t-1} + (1 - \beta) \vg_t - \vm_{t-1} = (1 - \beta) (\vg_t - \vm_{t-1})
\end{equation}
has small magnitude when $\beta$ is close to 1. \vspace{-0.43cm}
\paragraph{A New Product Norm.} Our proposed algorithm {\alg{MuonMax}} is regularized
steepest descent with respect to the following norm:
\begin{equation} \label{eq:muonmax_norm}
    \|\mW\|_{\text{MM}} := \textstyle\sqrt{ \big( \max_{\ell \in [L]} \|\mW^{\ell}\|_{2 \rightarrow 2} \big)^2 + \tfrac{\eta_m}{\eta_b} \|\boldsymbol{\theta}\|_{\text{ada2}}^2 }
\end{equation}
This norm comes from assigning $\|\cdot\|_{\text{ada2}}$ to the non-matrix parameters, spectral norm to the matrix parameters, then aggregating both using a combination of the $\ell_{\infty}$ and $\ell_2$ norms. We denote the corresponding product norm as $\|\cdot\|_{\text{hyb}}$, defined in \Eqref{eq:hyb_norm_def} of Appendix \ref{app:experiment_details}.
\psqueeze


\section{Model Truncation} \label{sec:truncation}
Beyond a more solid theoretical footing for Muon-type algorithms, our steepest descent
framework also offers practical benefits: techniques designed for SGD (or normalized
SGD) can now be easily adapted for Muon-type algorithms by generalizing to arbitrary
norms. In this section, we generalize \alg{Momo} \citep{momo} for steepest descent with
respect to arbitrary norms.

Recall that both variations of steepest descent are motivated by locally minimizing a
first-order Taylor approximation of the loss around the current iterate. \alg{Momo}
makes use of \emph{model truncation}~\citep{AsiDuchi2019PNAS}, which leverages knowledge
of a loss lower bound $F_*$ to construct a better approximation of the loss which is
more accurate than a Taylor approximation.
In \alg{Momo}, this model also incorporates information from the history of gradients
and losses through momentum.

Denote $\rho_{i,t} = (1 - \beta) \beta^{t-i}$, so that $\vm_t = \sum_{i=0}^t \rho_{i,t}
\vg_i$, and denote by $F_t(\vw_t)$ the minibatch loss at step $t$. Then for each $t$, we
can build a model of the loss around $\vw_t$ as a weighted average of first-order Taylor
approximations centered at each iterate $\vw_i$:
\begin{align}
    F(\vw) &\approx \textstyle \sum_{i=0}^t \rho_{t,i} \left( F_i(\vw_i) + \langle \vg_i, \vw - \vw_i \rangle \right) \\
    &= \textstyle \sum_{i=0}^t \rho_{t,i} \left( F_i(\vw_i) + \langle \vg_i, \vw_t - \vw_i \rangle \right) + \sum_{i=0}^t \rho_{t,i} \langle \vg_i, \vw - \vw_t \rangle \\
    &= \textstyle \tilde{F}_t + \langle \vm_t, \vw - \vw_t \rangle,
\end{align}
where on the last line we denoted $\tilde{F}_t := \sum_{i=0}^t \rho_{t,i} \left(
F_i(\vw_i) + \langle \vg_i, \vw_t - \vw_i \rangle \right)$. Since $F(\vw) \geq F_*$ for
all $\vw$, we can improve our model by truncating, or clipping, it at $F_*$:
$$ \textstyle
    F(\vw) \approx \max \left( \tilde{F}_t + \langle \vm_t, \vw - \vw_t \rangle, F_* \right).
    $$
Our truncated steepest descent methods, shown below, arise from minimizing this
truncated model either with a norm ball constraint or with squared norm regularization.

\begin{restatable}{proposition}{propmomocsd}[Constrained Momo]\label{prop:momo_csd}
The ball constrained truncated model update is given by
\begin{align}
    \vw_{t+1} &= \argmin_{\|\vw - \vw_t\| \leq \eta} \left\{ \max \left( \tilde{F}_t + \langle \vm_t, \vw - \vw_t \rangle, F_* \right) \right\} \\
&=    \vw_t + \min \left( \eta, \tfrac{\tilde{F}_t - F_*}{\|\vm_t\|_*} \right) {\sf LMO}(\vm_t)
\end{align}
\end{restatable}

The $\argmin$ above can have multiple solutions: we take the one that has
minimal distance to $\vw_t$.

\begin{restatable}{proposition}{propmomorsd}[Regularized Momo]\label{prop:momo_rsd}
The regularized truncated model update is given by
\begin{align}
    \vw_{t+1} &= \argmin_{\vw} \left\{ \max \left( \tilde{F}_t + \langle \vm_t, \vw - \vw_t \rangle, F_* \right) + \tfrac{1}{2 \eta} \|\vw - \vw_t\|^2 \right\}\\
    &= \vw_t + \min \left( \eta, \tfrac{\tilde{F}_t - F_*}{\|\vm_t\|_*^2} \right) \|\vm_t\|_* {\sf LMO}(\vm_t)
\end{align}
\end{restatable}

The term $\tilde{F}_t$ relies on the history of previous gradients and losses, 
but it can be computed with a single scalar running average.
Pseudocode  for both \alg{Momo} variations is shown in Algorithm \ref{alg:momo} of
Appendix \ref{app:momo_proofs}.

Now that we have shown how to use \alg{Momo} with respect to any norm, we can
immediately combine \alg{Momo} with any steepest descent algorithm in our framework,
including \alg{MuonAdam}. For example, our proposed algorithm \alg{MuonMax}-\alg{Momo}
(Algorithm \ref{alg:muonmax} in Appendix \ref{app:momo_proofs}) can be written as
Regularized \alg{Momo} w.r.t. $\|\cdot\|_{\text{MM}}$ (defined in
\Eqref{eq:muonmax_norm}) with stale dual norm approximations.

\begin{restatable}{proposition}{propmuonmax}[\alg{MuonMax}-\alg{Momo}] \label{prop:muonmax}
Regularized \alg{Momo} with respect to the norm $\|\cdot\|_{\text{MM}}$ as defined in~\eqref{eq:muonmax_norm} has the following closed form:
\begin{equation} \label{eq:muonmax_momo}
\begin{aligned}
    d_t &= \textstyle \sqrt{ \Big( \sum_{\ell=1}^L \|\mM_t^{\ell}\|_{\text{nuc}} \Big)^2 + \tfrac{\eta_b}{\eta_m} \Big\| \tfrac{\vm_t^{\theta}}{\sqrt{\sqrt{\vv_t^{\theta}} + \epsilon}} \Big\|_2^2 } \\
    \mW_{t+1}^{\ell} &= \textstyle \mW_t^{\ell} - \min \Big\{ \eta_m, \tfrac{\tilde{F}_t - F_*}{d_t^2} \Big\} \Big( \sum_{j=1}^L \|\mM_t^j\|_{\text{nuc}} \Big) \text{polar}(\mM_t^{\ell}) \\
    \boldsymbol{\theta}_{t+1} &=  \textstyle\boldsymbol{\theta}_t - \min \Big\{\eta_b, \tfrac{\eta_b}{\eta_m} \tfrac{\tilde{F}_t - F_*}{d_t^2} \Big\} \tfrac{\vm_t^{\theta}}{\sqrt{\vv_t^{\theta}} + \epsilon}.
\end{aligned}
\end{equation}
\end{restatable}
The update in Proposition \ref{prop:muonmax} matches that of Algorithm
\ref{alg:muonmax} except for the use of stale dual norms. \psqueeze

\section{Experiments}
Here we provide a comprehensive evaluation of optimizers arising from our
steepest descent framework for training language models. We start by tuning and
evaluating 36 optimizer variations arising from different choices of normalization,
product norm, norm for the non-matrix parameters, and whether to use
model truncation. For this initial phase of evaluating all variations, we use 1B tokens
from the FineWeb dataset to train a GPT2-Small model with 124M params (Section
\ref{sec:fineweb}). We take the four best performing methods (\alg{MuonAdam}, \alg{Scion},
\alg{MuonAdam}-\alg{Momo}, \alg{MuonMax}-\alg{Momo}) and evaluate them for a GPT2-Large
model with 774M params on the SlimPajama dataset (Section \ref{sec:slim_pajama}), first
by thoroughly tuning all four algorithms with 1B tokens, then running a final evaluation
of \alg{Muon} and \alg{MuonMax}-\alg{Momo} with 6B tokens. Finally, in Section
\ref{sec:ablation} we perform two ablation studies: we examine the sensitivity of
\alg{Momo} variants to the choice of the loss lower bound  $F_*$, and we evaluate
the effect of stale nuclear norm approximations on final loss and wall-clock time.

\subsection{FineWeb Dataset} \label{sec:fineweb}
To identify the strongest methods within our framework, we thoroughly tune and evaluate
36 variations that arise from mixing and matching settings for the following design
choices: constrained vs regularized steepest descent, product norm ($\|\cdot\|_{\infty},
\|\cdot\|_2, \|\cdot\|_{\text{hyb}}$), norm for parameters besides hidden weight
matrices ($\|\cdot\|_{\infty}, \|\cdot\|_{\text{ada}\infty}, \|\cdot\|_{\text{ada2}}$),
and whether to apply model truncation.\psqueeze

\paragraph{Setup.}
For all variations, we run one epoch of training with 1B tokens from the FineWeb dataset
\citep{penedo2024fineweb}, using the GPT-2 Small architecture (124M params) from
modded-nanogpt \citep{modded_nanogpt_2024}. Each algorithm in our framework has two
learning rates: $\eta_m$ for the hidden weight matrices (which we call the Muon
learning rate) and $\eta_b$ for everything else (which we call the base learning rate).
Due to the computational cost of performing a double grid search, we opt to tune with an
iterated grid search; for each algorithm, we fix $\eta_b$ while tuning $\eta_m$, then
fix $\eta_m$ at the tuned value while tuning $\eta_b$. See Appendix
\ref{app:experiment_details} for a complete specification of the tuning protocol and
other implementation details. For all \alg{Momo} variations, we set the lower bound $F_*
= 3.2$, and conduct a sensitivity analysis of this hyperparameter in Section
\ref{sec:ablation}.\psqueeze

\paragraph{Results.}
The final loss for each variation is shown in Tables \ref{tab:all_fineweb_losses} and
\ref{tab:all_fineweb_losses_truncated} of Appendix \ref{app:experiment_results}. For the
best performing variations (\alg{MuonAdam}, \alg{Scion}, \alg{MuonAdam}-\alg{Momo}, and
\alg{MuonMax}-\alg{Momo}), we additionally evaluate the sensitivity to learning rate
tuning by running each algorithm with LRs $(\rho \eta_m, \rho \eta_b)$, where $(\eta_m,
\eta_b)$ are the previously tuned LRs and $\rho$ varies over $\{0.03, 0.1, 0.3, 1, 3,
10, 30, 100\}$, with three random seeds each (Figure \ref{fig:fineweb_final}). Table
\ref{tab:final_fineweb_losses} in Appendix \ref{app:experiment_results} gives the mean
and standard deviation of final validation loss for each algorithm with tuned LRs. For
these runs, \alg{MuonAdam}-\alg{Momo} and \alg{MuonMax}-\alg{Momo} use stale
nuclear norms.

In Figure \ref{fig:fineweb_final}, we see that \alg{MuonAdam} and
\alg{MuonAdam}-\alg{Momo} achieve the smallest loss among all variations, though
\alg{MuonAdam} is much more sensitive to the learning rate. Both
\alg{MuonAdam}-\alg{Momo} and \alg{MuonMax}-\alg{Momo} enjoy a much wider range of
competitive learning rates compared with \alg{MuonAdam} and \alg{Scion}; for this search
range, the proportion of LRs yielding loss less than 3.65 is 25\% for \alg{MuonAdam} and
\alg{Scion}, 50\% for \alg{MuonMax}-\alg{Momo}, and 62.5\% for
\alg{MuonAdam}-\alg{Momo}. Also, Table \ref{tab:final_fineweb_losses} (Appendix
\ref{app:experiment_results}) shows that both of our \alg{Momo} methods achieve a
smaller variation in loss across random seeds compared with \alg{MuonAdam} and
\alg{Scion}.


\begin{figure}[t]
\centering
\begin{subfigure}[b]{0.49\textwidth}
\centering
\includegraphics[width=\textwidth]{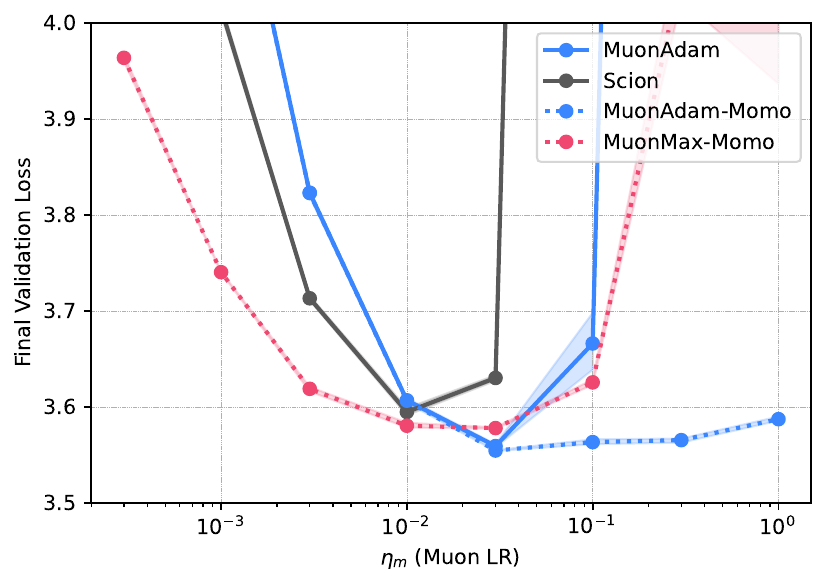}
\caption{FineWeb1B (GPT2-Small).}
\label{fig:fineweb_final}
\end{subfigure}
\hfill
\begin{subfigure}[b]{0.49\textwidth}
\centering
\includegraphics[width=\textwidth]{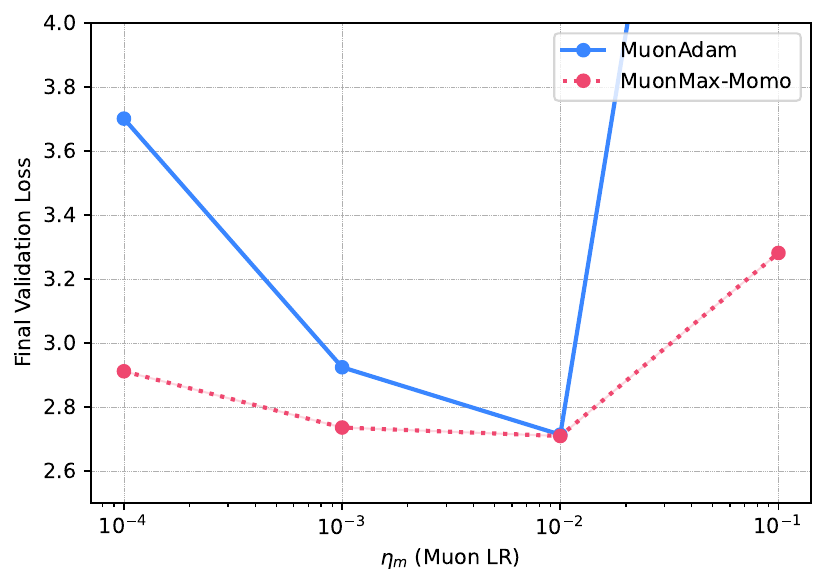}
\caption{SlimPajama6B (GPT2-Large).}
\label{fig:slim_pajama6B}
\end{subfigure}
\caption{Final validation loss with varying learning rates on FineWeb1B (left) and
SlimPajama6B (right). Our \alg{MuonAdam}-\alg{Momo} and \alg{MuonMax}-\alg{Momo} have
wider basins than \alg{MuonAdam} and \alg{Scion}, indicating increased robustness to
learning rate tuning.}
\label{fig:fineweb1B_slimpajama6B}
\end{figure}

\subsection{SlimPajama Dataset} \label{sec:slim_pajama}
Having identified \alg{MuonAdam}, \alg{Scion}, \alg{MuonAdam}-\alg{Momo}, and
\alg{MuonMax}-\alg{Momo} as the strongest variations, we evaluate these methods for
training the GPT2-Large architecture (774M params) using the SlimPajama dataset
\citep{cerebras2023slimpajama}. We first evaluate all four algorithms for one epoch with
1B tokens, then evaluate \alg{MuonAdam} and \alg{MuonMax}-\alg{Momo} for one epoch with
6B tokens. \psqueeze

\paragraph{Setup.} Most aspects of training are the same as in Section \ref{sec:fineweb},
the main difference being the tuning protocol. To tune the two learning rates $\eta_m$
and $\eta_b$, we run a double grid search for each algorithm, varying $\eta_m \in
\{$1e-4, 1e-3, 1e-2, 1e-1$\}$ and $\eta_b \in \{$1e-5, 1e-4, 1e-3, 1e-2$\}$ for a total
of 16 settings per algorithm. For the \alg{Momo} variants, we set the lower bound $F_* =
2.8$ when training with 1B tokens and $F_* = 2.0$ when training with 6B tokens. We did
not tune $F_*$, and based on the sensitivity analysis in Section \ref{sec:ablation}, we
expect that this hyperparameter does not have a large effect on final performance for
tuned learning rates.  \psqueeze

\paragraph{Results.}
Figure \ref{fig:slim_pajama1B} shows the final loss of each method with LRs $(\rho
\eta_m, \rho \eta_b)$, where $(\eta_m, \eta_b)$ are tuned LRs and $\rho \in \{$1e-2,
1e-1, 1, 1e1, 1e2, 1e3$\}$. The sensitivity of each method with respect to both learning
rates is shown for the full 2D grid in Figure \ref{fig:slim_pajama1B_heatmap} of
Appendix \ref{app:experiment_results}. We see in Figure \ref{fig:slim_pajama1B} that
\alg{MuonMax}-\alg{Momo} achieves the lowest loss of all methods, and that both
\alg{Momo} variations are extremely robust to the choice of learning rates. Both
\alg{MuonAdam} and \alg{Scion} have quite narrow sensitivity curves, that is, shifting
the optimal learning rates by a factor of 10 in either direction creates a large
increase in final loss. In comparison, the final loss of \alg{MuonMax}-\alg{Momo}
remains between 3.13 and 3.24 even as $\eta_m$ varies over five orders of magnitude from
1e-3 to 10.

We see similar robustness of \alg{MuonMax}-\alg{Momo} when scaling up to 6B tokens. Due
to the cost of re-tuning learning rates, we reuse the ratio $\eta_m/\eta_b$ of the tuned
learning rates from 1B training, and vary $\eta_m \in \{$1e-4, 1e-3, 1e-2, 1e-1$\}$ for
\alg{MuonAdam} and \alg{MuonMax}-\alg{Momo}. Figure \ref{fig:slim_pajama6B} shows that
\alg{MuonMax}-\alg{Momo} achieves a lower loss than \alg{MuonAdam} for every setting in
this range, and generally exhibits less variation in the loss as the learning rates are
shifted from their optimal values.


\subsection{Ablations} \label{sec:ablation}
To probe the behavior of our proposed methods, we perform two ablation studies: (1) we
evaluate how the choice of loss lower bound $F_*$ affects the final validation loss of
\alg{MuonAdam}-\alg{Momo} and \alg{MuonMax}-\alg{Momo}; (2) we evaluate the effect of using
stale nuclear norm approximations on the final validation loss and wall-clock time per
iteration for several methods in our framework. In this section, we use the same setup
as in Section \ref{sec:fineweb} (GPT2-Small, FineWeb dataset, 1B tokens). \psqueeze

\paragraph{Sensitivity Analysis of $F_*$.} Figure \ref{fig:lb_sensitivity} shows the
final loss of \alg{MuonAdam}-\alg{Momo} and \alg{MuonMax}-\alg{Momo} with various
$\eta_m$, as the loss lower bound $F_*$ varies over $\{0, 1.6, 2.4, 2.8, 3.2\}$. We see
that the choice of $F_*$ makes the biggest difference when $\eta_m$ is larger than the
optimal value. For \alg{MuonAdam}-\alg{Momo}, the final loss is nearly identical for all
values of $F_*$ when $\eta_m \leq 0.1$. \alg{MuonMax}-\alg{Momo} is somewhat more
sensitive to the choice of $F_*$, but even the aggressive lower bound of $F_* = 0.0$
achieves 3.61 loss compared to the 3.58 optimum achieved with $F_* = 3.2$. \psqueeze

\begin{figure}[t]
\centering
\includegraphics[width=0.49\textwidth]{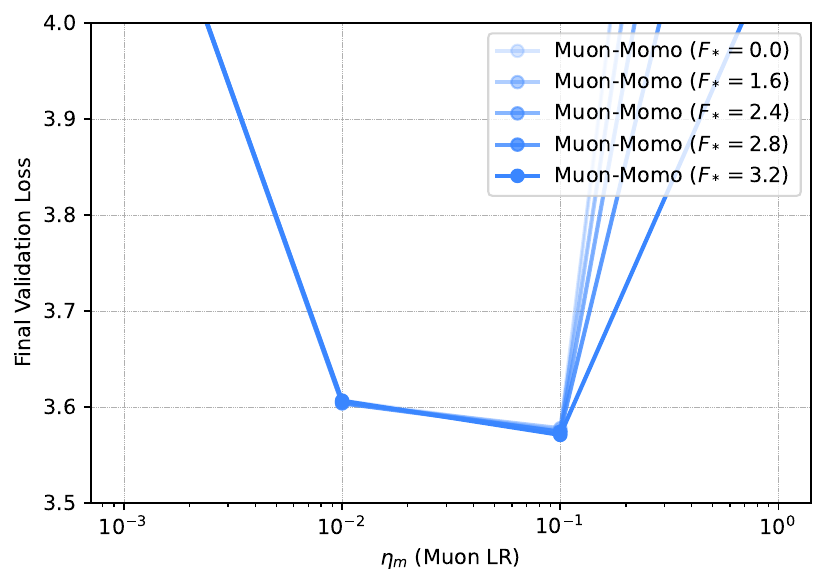}
\includegraphics[width=0.49\textwidth]{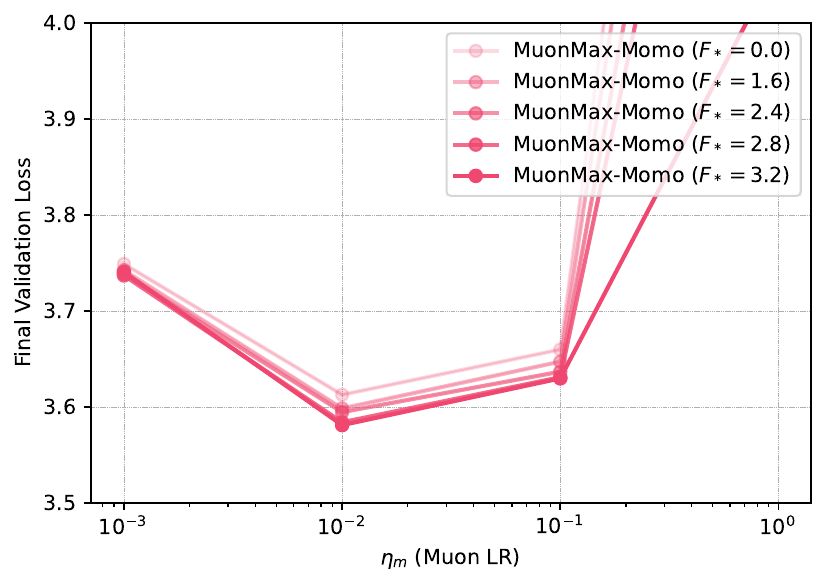}
\caption{Sensitivity to loss lower bound $F_*$ for model truncation (Fineweb1B).}
\label{fig:lb_sensitivity}
\end{figure}

\paragraph{Effect of Stale Approximation.} Table \ref{tab:stale} shows the final
validation loss and per-step wall-clock times of four methods (with tuned LRs) with and
without stale nuclear norm approximations. We see that in all cases, the stale
approximation increases the lost by at most $0.004$, while sometimes even decreasing it.
We therefore conclude that this approximation does not noticeably affect the final loss
for these tuned algorithms, although it does afford a speedup; for
\alg{MuonMax}-\alg{Momo}, the additional wall-clock time compared to \alg{MuonAdam} is
reduced from 11\% to 5\%. \psqueeze

\begin{table}[t]
\centering
\begin{tabular}{lccccc}
\toprule
& \alg{MuonAdam} & \alg{MuonMax} & \alg{MuonAdam}-\alg{Momo} & \alg{Scion}-\alg{Momo} & \alg{MuonMax}-\alg{Momo} \\
\midrule
Original & $3.604$ ($1 \times$) & $3.791$ ($1.09 \times$) & $3.551$ ($1.10 \times)$ & $3.592$ ($1.08 \times$) & $3.576$ ($1.11 \times$) \\
Stale & - & $3.768$ ($1.04 \times$) & $3.554$ ($1.04 \times$) & $3.590$ ($1.02 \times$) & $3.580$ ($1.05 \times$) \\
\bottomrule
\end{tabular}
\caption{Effect of stale nuclear norm approximation on final loss and wall-clock time per-iteration compared to \alg{MuonAdam}, which has no stale variant because it does not involve nuclear norms.}
\label{tab:stale}
\end{table}

\newpage

\bibliography{references.bib}
\bibliographystyle{iclr2026_conference.bst}

\newpage
\appendix
\tableofcontents
\newpage

\section{Proofs from Section \ref{sec:steepest}} \label{app:steepest_proofs}
In what follows, for a norm denoted by a subscript such as $\|\cdot\|_{\infty}$,
we will sometimes replace ${\sf LMO}_{\|\cdot\|_{\infty}}$ with ${\sf
LMO}_{\infty}$.

\propcsd*
\begin{proof}[Proof of Proposition \ref{prop:csd}]
Denoting $r = \|\vw - \vw_t\|$ and $\boldsymbol{\Delta} = (\vw - \vw_t) / \|\vw -
\vw_t\|$, we can change variables in the optimization problem from \Eqref{eq:csd_opt},
yielding $\vw_{t+1} = \vw_t + r_t \boldsymbol{\Delta}_t$, where
\begin{equation}
    (r_t, \boldsymbol{\Delta}_t) = \argmin_{r \in [0, \eta], \|\boldsymbol{\Delta}\|=1} \left\{ r \langle \vm_t, \boldsymbol{\Delta} \rangle \right\},
\end{equation}
which can be separated into
\begin{equation}
    \boldsymbol{\Delta}_t = \argmin_{\|\boldsymbol{\Delta}\|=1} \langle \vm_t, \boldsymbol{\Delta} \rangle = {\sf LMO}(\vm_t),
\end{equation}
and
\begin{equation}
    r_t = \argmin_{r \in [0, \eta]} \left\{ r \langle \vm_t, \Delta_t \rangle \right\} = \argmin_{r \in [0, \eta]} \left\{ r \langle \vm_t, {\sf LMO}(\vm_t) \rangle \right\} = \argmin_{r \in [0, \eta]} \left\{ -r \|\vm_t\|_* \right\} = \eta,
\end{equation}
so $\vw_{t+1} = \vw_t + \eta {\sf LMO}(\vm_t)$.
\end{proof}

\proprsd*
\begin{proof}[Proof of Proposition \ref{prop:rsd}]
For the optimization problem from \Eqref{eq:rsd_opt}, we use the same change of
variables as in the proof of Proposition \ref{prop:csd}: $r = \|\vw - \vw_t\|$ and
$\boldsymbol{\Delta} = (\vw - \vw_t) / \|\vw - \vw_t\|$. Therefore $\vw_{t+1} = \vw_t +
r_t \boldsymbol{\Delta}_t$, where
\begin{align}
    (r_t, \boldsymbol{\Delta}_t) = \argmin_{r \geq 0, \|\boldsymbol{\Delta}\|=1} \left\{ r \langle \vm_t, \boldsymbol{\Delta} \rangle + \tfrac{r^2}{2 \eta} \right\},
\end{align}
which can be separated into
\begin{equation}
    \boldsymbol{\Delta}_t = \argmin_{\|\boldsymbol{\Delta}\|=1} \langle \vm_t, \boldsymbol{\Delta} \rangle = {\sf LMO}(\vm_t),
\end{equation}
and
\begin{align}
    r_t &= \argmin_{r \geq 0} \left\{ r \langle \vm_t, \boldsymbol{\Delta_t} \rangle + \tfrac{r^2}{2 \eta} \right\} \\
    &= \argmin_{r \geq 0} \left\{ r \langle \vm_t, {\sf LMO}(\vm_t) \rangle + \tfrac{r^2}{2 \eta} \right\} \\
    &= \argmin_{r \geq 0} \left\{ -r \|\vm_t\|_* + \tfrac{r^2}{2 \eta} \right\} \\
    &= \eta \|\vm_t\|_*,
\end{align}
so that $\vw_{t+1} = \vw_t + \eta \|\vm_t\|_* {\sf LMO}(\vm_t)$.
\end{proof}

\lemproductnorm*
\begin{proof}[Proof of Lemma 3.3]
To show that $h$ is a norm, we only need to show that
\begin{enumerate}
    \item $h(\vw_1, \ldots, \vw_n) \geq 0$ for all $\vw_1, \ldots, \vw_n$,
    \item $h(\vw_1, \ldots, \vw_n) = 0$ if and only if $(\vw_1 \ldots, \vw_n) = \vzero$,
    \item $h(\lambda \vw_1, \ldots, \lambda \vw_n) = |\lambda| h(\vw_1, \ldots, \vw_n)$ for all $\lambda \in \R, \vw_1, \ldots, \vw_n$,
    \item $h(\vw_1 + \vv_1, \ldots, \vw_n + \vv_n) \leq h(\vw_1, \ldots, \vw_n) + h(\vv_1, \ldots, \vv_n)$ for all $\vw_1, \vv_1, \ldots, \vw_n, \vv_n$.
\end{enumerate}
All of these properties hold immediately from the definition of $h = f \circ (g_1,
\ldots, g_n)$ together with repeated applications of the norm properties of $f$ and
$g_1, \ldots, g_n$.

From the definition of the dual norm,
\begin{align}
    h_*(\vw_1, \ldots, \vw_n) &= \max \left\{ \sum_{i=1}^n \langle \vw_i, \vv_i \rangle \;\middle|\; h(\vv_1, \ldots, \vv_n) = 1 \right\} \\
    &= \max \left\{ \sum_{i=1}^n \langle \vw_i, \vv_i \rangle \;\middle|\; f(g_1(\vv_1), \ldots, g_n(\vv_n)) = 1 \right\}. \label{eq:product_inter}
\end{align}
We use a change of variables $\vu_i = \vv_i/g_i(\vv_i)$ and $r_i = g_i(\vv_i)$, which
separates the update direction $\vu_i$ (with unit norm) from the update norm $r_i$. So
\Eqref{eq:product_inter} is equivalent to
\begin{align}
    h_*(\vw_1, \ldots, \vw_n) &= \max \left\{ \sum_{i=1}^n r_i \langle \vw_i, \vu_i \rangle \;\middle|\; f(r_1, \ldots, r_n) = 1 \right\} \label{eq:product_dual_inter}
\end{align}
Note that the condition $f(r_1, \ldots, r_n) = 1$ does not involve $\vu_i$, so each term
$r_i \langle \vw_i, \vu_i \rangle$ is maximized when
\begin{equation}
    \vu_i = \argmax_{g_i(\vz_i)=1} \langle \vw_i, \vz_i \rangle = -{\sf LMO}_{g_i}(\vw_i),
\end{equation}
with maximum value $\langle \vw_i, \vu_i \rangle =  g_{i,*}(\vw_i)$. Using this in
\Eqref{eq:product_dual_inter} gives
\begin{align}
    h_*(\vw_1, \ldots, \vw_n) &= \max \left\{ \sum_{i=1}^n r_i g_{i,*}(\vw_i) \;\middle|\; f(r_1, \ldots, r_n) = 1 \right\}.
\end{align}
Denoting $\vr = (r_1, \ldots, r_n)$ and $\vs = (g_{1,*}(\vw_1), \ldots,
g_{n,*}(\vw_n))$, this is equivalent to
\begin{align}
    h_*(\vw_1, \ldots, \vw_n) &= \max \left\{ \langle \vr, \vs \rangle \;\middle|\; f(\vr) = 1 \right\} \\
    &= f_*(\vs),
\end{align}
which gives us the dual norm $h_*$.

To obtain ${\sf LMO}_h$, we only need to look at the value of the variables which
achieved the maximum in the above derivation:
\begin{align}
    \vu_i = -{\sf LMO}_{g_i}(\vw_i), \quad \text{and} \quad \vr = {\sf LMO}_f(g_{1,*}(\vw_1), \ldots, g_{n,*}(\vw_n))
\end{align}
so that
\begin{equation}
    \vv_i = r_i {\sf LMO}_f(g_{1,*}(\vw_1), \ldots, g_{n,*}(\vw_n))
\end{equation}
maximizes $\sum_{i=1}^n \langle \vw_i, \vv_i \rangle$ subject to $h(\vv_1, \ldots,
\vv_n)=1$. Note that ${\sf LMO}_h(\vw_1, \ldots, \vw_n)$ is exactly the minimizer of
$\sum_{i=1}^n \langle \vw_i, \vv_i \rangle$ subject to the same norm constraint; since
$\sum_{i=1}^n \langle \vw_i, \vv_i \rangle$ is linear in $\vv_i$, the minimizer is the
negative of the maximizer. Therefore
\begin{equation}
    {\sf LMO}_h(\vw_1, \ldots, \vw_n) = -(r_1 {\sf LMO}_{g_1}(\vw_1), \ldots, r_n {\sf LMO}_{g_n}(\vw_n)).
\end{equation}
\end{proof}

The following lemma will be useful later for quickly computing duals and ${\sf LMO}$s of
various norms.

\begin{lemma} \label{lem:linear_norm}
For any norm $\|\cdot\|$ on $\R^d$ full rank matrix $\mD \in \R^{d \times d}$, the norm
defined by $\|\vv\|_{\mD} = \|\mD \vv\|$ has
\begin{align}
    {\sf LMO}_{\|\cdot\|_{\mD}}(\vv) &= \mD^{-1} {\sf LMO}_{\|\cdot\|}(\mD^{-T} \vv), \\
    \|\vv\|_{\mD,*} &= \left\| \mD^{-T} \vv \right\|_*.
\end{align}
\end{lemma}

\begin{proof}
The fact that $\|\cdot\|_{\mD}$ is a norm follows immediately from the norm properties
of $\|\cdot\|$ together with the fact that $\mD$ is full rank. For the dual norm,
\begin{equation}
    \|\vv\|_{\mD,*} = \max_{\|\vu\|_{\mD}=1} \langle \vv, \vu \rangle = \max_{\|\mD \vu\|=1} \langle \vv, \vu \rangle
\end{equation}
and a change of variables $\vz = \mD \vu$ yields
\begin{equation}
    \|\vv\|_{\mD,*} = \max_{\|\vz\|=1} \langle \vv, \mD^{-1} \vz \rangle = \max_{\|\vz\|=1} \langle \mD^{-T} \vv, \vz \rangle = \left\| \mD^{-T} \vv \right\|_*.
\end{equation}
For the ${\sf LMO}$, we can look at the value of the variables that maximize the inner
product in the above:
\begin{equation}
    \vz = \argmax_{\|\vz\|=1} \langle \mD^{-T} \vv, \vz \rangle = -{\sf LMO}_{\|\cdot\|}(\mD^{-T} \vv).
\end{equation}
Returning to the $\vu$ variable then gives
$$\vu = \mD^{-1} \vz = -\mD^{-1} {\sf LMO}_{\|\cdot\|}(\mD^{-T} \vv)$$ which maximizes
$\langle \vv, \vu \rangle$ subject to $\|\vu\|_{\mD} = 1$. Since $\langle \vv, \vu
\rangle$ is linear in $\vu$, the minimizer of $\langle \vv, \vu \rangle$ under the norm
constraint $\|\vu\|_{\mD} = 1$ is exactly the negative of the maximizer under the same
constraint. So
\begin{equation}
    {\sf LMO}_{\|\cdot\|_{\mD}}(\vv) = \argmin_{\|\vu\|_{\mD}=1} \langle \vv, \vu \rangle = \mD^{-1} {\sf LMO}_{\|\cdot\|}(\mD^{-T} \vv)
\end{equation}
\end{proof}

\propadaminfty*
\begin{proof}[Proof of Proposition \ref{prop:adam_infty}]
Let $\mD = \text{Diag} \left( \tfrac{\sqrt{\vv_t} + \epsilon}{|\vm_t|} \right)$, so that
$\|\boldsymbol{\theta}\|_{\text{ada}\infty} = \|\mD \boldsymbol{\theta}\|_{\infty}$.
Then by Proposition \ref{prop:csd}, one step of CSD w.r.t.
$\|\cdot\|_{\text{ada}\infty}$ is given by
\begin{align}
    \boldsymbol{\theta}_{t+1} &= \boldsymbol{\theta}_t + \eta {\sf LMO}_{\text{ada}\infty}(\vm_t) \\
    &\Eqmark{i}{=} \boldsymbol{\theta}_t + \eta \mD^{-1} {\sf LMO}_{\infty}(\mD^{-T} \vm_t) \\
    &\Eqmark{ii}{=} \boldsymbol{\theta}_t - \eta \mD^{-1} \sign(\mD^{-T} \vm_t) \\
    &= \boldsymbol{\theta}_t - \eta \text{Diag} \left( \tfrac{|\vm_t|}{\sqrt{\vv_t} + \epsilon} \right) \sign \left( \text{Diag} \left( \tfrac{|\vm_t|}{\sqrt{\vv_t} + \epsilon} \right) \vm_t \right) \\
    &= \boldsymbol{\theta}_t - \eta \tfrac{|\vm_t|}{\sqrt{\vv_t} + \epsilon} \odot \sign(\vm_t) \\
    &= \boldsymbol{\theta}_t - \eta \tfrac{\vm_t}{\sqrt{\vv_t} + \epsilon},
\end{align}
where $(i)$ uses Lemma \ref{lem:linear_norm} and $(ii)$ uses ${\sf
LMO}_{\infty}(\vv) = -\sign(\vv)$.
\end{proof}

\propadamtwo*
\begin{proof}[Proof of Proposition \ref{prop:adam_2}]
Let $\mD = \text{Diag} \left( \sqrt{\sqrt{\vv_t} + \epsilon} \right)$, so that
$\|\boldsymbol{\theta}\|_{\text{ada}2} = \|\mD \boldsymbol{\theta}\|_2$. Then by
Proposition \ref{prop:rsd}, one step of RSD w.r.t. $\|\cdot\|_{\text{ada}2}$ is given by
\begin{align}
    \boldsymbol{\theta}_{t+1} &= \boldsymbol{\theta}_t + \eta \|\vm_t\|_{\text{ada}2,*} {\sf LMO}_{\text{ada}2}(\vm_t) \\
    &\Eqmark{i}{=} \boldsymbol{\theta}_t + \eta \|\mD^{-T} \vm_t\|_{2,*} \mD^{-1} {\sf LMO}_2(\mD^{-T} \vm_t) \\
    &\Eqmark{ii}{=} \boldsymbol{\theta}_t - \eta \|\mD^{-T} \vm_t\|_2 \mD^{-1} \tfrac{\mD^{-T} \vm_t}{\|\mD^{-T} \vm_t\|_2} \\
    &= \boldsymbol{\theta}_t - \eta \mD^{-1} \mD^{-T} \vm_t \\
    &= \boldsymbol{\theta}_t - \eta \text{Diag} \left( \tfrac{1}{\sqrt{\vv_t}+\epsilon} \right) \vm_t \\
    &= \boldsymbol{\theta}_t - \eta \tfrac{\vm_t}{\sqrt{\vv_t}+\epsilon},
\end{align}
where $(i)$ uses Lemma \ref{lem:linear_norm} and $(ii)$ uses ${\sf
LMO}_2(\vv) = -\vv/\|\vv\|_2$.
\end{proof}

For reference, we include the pseudocode for \alg{MuonAdam} (\alg{Muon} side-by-side with \alg{Adam})
in Algorithm \ref{alg:muonadam}.

\begin{algorithm}[t]
    \caption{\alg{MuonAdam}: where $\mW^1, \ldots, \mW^L$ are the weight matrices,  \\ \phantom{Algorithm 1:} and 
    $\boldsymbol{\theta}$ are all other parameters flattened into a vector.}
    \label{alg:muonadam}
    \AlgIO{Inputs}{$\mW_0^1, \ldots, \mW_0^L, \boldsymbol{\theta}_0$, learning rates $\eta_b,\eta_m$, EMA parameters $\beta, \beta_1, \beta_2$}
    \begin{algorithmic}[1]
        \For{$t = 0, 1, \ldots, T-1$}
            \State $(\mG_t^1, \ldots, \mG_t^L, \vg_t^{\theta}) \gets \text{backward}(\mW_t^1, \ldots \mW_t^L, \boldsymbol{\theta}_t)$
            \For{$\ell = 1, \ldots, L$}
                \State $\mM_t^{\ell} = \beta \mM_{t-1}^{\ell} + (1 - \beta) \mG_t^{\ell}$
                \State $\mW_{t+1}^{\ell} \gets \mW_t^{\ell} - \eta_m \text{polar}(\mM_t^{\ell})$
            \EndFor
            \State $\vm_t^{\theta} = \beta_1 \vm_{t-1}^{\theta} + (1 - \beta_1) \vg_t^{\theta}$
            \State $\vv_t^{\theta} = \beta_2 \vv_{t-1}^{\theta} + (1 - \beta_2) \vg_t^{\theta} \odot \vg_t^{\theta}$
            \State $\boldsymbol{\theta}_{t+1} = \boldsymbol{\theta}_t - \eta_b \tfrac{\vm_t^{\theta}}{\sqrt{\vv_t^{\theta}} + \epsilon}$
        \EndFor
    \end{algorithmic}
\end{algorithm}

\propmuon*
\begin{proof}[Proof of Proposition \ref{prop:muon}]
By Proposition \ref{prop:csd}, one step of CSD w.r.t. $\|\cdot\|_{\text{muon}}$ can be
written as
\begin{equation} \label{eq:muon_update_inter}
    \mW_{t+1} = \mW_t + \eta_m {\sf LMO}_{\text{muon}}(\mM_t),
\end{equation}
where $\mM_t$ is the momentum buffer for all network parameters, i.e. it is the
concatenation of the momentum buffers of each parameter:
\begin{equation}
    \mM_t = (\mM_t^1, \ldots, \mM_t^L, \vm_t^{\theta}).
\end{equation}
Denote $\lambda = \eta_b / \eta_m$. To compute the ${\sf LMO}$ term, we can rewrite $\|\mW\|_{\text{muon}}$ as
\begin{equation}
    \|\mW\|_{\text{muon}} = \max \left( \|\mW^1\|_{2 \rightarrow 2}, \ldots \|\mW^L\|_{2 \rightarrow 2}, \tfrac{1}{\lambda} \|\boldsymbol{\theta}\|_{\text{ada}\infty} \right),
\end{equation}
so that $\|\cdot\|_{\text{muon}}$ can be written as the composition (as in Lemma \ref{lem:product_norm})
\begin{equation}
    \|\mW\|_{\text{muon}} = f(g_1(\mW^1), \ldots g_L(\mW^L), g_{L+1}(\boldsymbol{\theta})),
\end{equation}
with $g_i(\mM) = \|\mM\|_{2 \rightarrow 2}$ for $i \in [L]$,
$g_{L+1}(\boldsymbol{\theta}) = \|\boldsymbol{\theta}\|_{\text{ada}\infty}$, and $f(\vv)
= \|\mD \vv\|_{\infty}$, where $\mD = \text{Diag}(1, \ldots, 1, 1/\lambda) \in \R^{(L+1)
\times (L+1)}$. Therefore, by Lemma \ref{lem:product_norm}, the update in
\Eqref{eq:muon_update_inter} is equivalent to
\begin{equation} \label{eq:muon_update_inter_2}
\begin{aligned}
    \mW_{t+1}^{\ell} &= \mW_t^{\ell} + \eta_m \phi_{\ell} {\sf LMO}_{2 \rightarrow 2}(\mM_t^{\ell}) \\
    \boldsymbol{\theta}_{t+1} &= \boldsymbol{\theta}_t + \eta_m \phi_{L+1} {\sf LMO}_{\text{ada}\infty}(\vm_t^{\theta}),
\end{aligned}
\end{equation}
where $\boldsymbol{\phi} = -{\sf LMO}_f(\|\mM_t^1\|_{\text{nuc}}, \ldots, \|\mM_t^L\|_{\text{nuc}}, \|\vm_t^{\theta}\|_{\text{ada}\infty,*})$. We know ${\sf LMO}_{2 \rightarrow 2}(\mM) = -\text{polar}(\mM)$, and we proved in Proposition \ref{prop:adam_infty} that
\begin{equation}
    {\sf LMO}_{\text{ada}\infty}(\vv) = -\tfrac{|\vm_t^{\theta}|}{\sqrt{\vv_t^{\theta}} + \epsilon} \sign(\vv),
\end{equation}
so the ${\sf LMO}$ terms in \Eqref{eq:muon_update_inter_2} can be simplified as
\begin{equation}
\begin{aligned} \label{eq:muon_update_inter_3}
    \mW_{t+1}^{\ell} &= \mW_t^{\ell} - \eta_m \phi_{\ell} \text{polar}(\mM_t^{\ell}) \\
    \boldsymbol{\theta}_{t+1} &= \boldsymbol{\theta}_t - \eta_m \phi_{L+1} \tfrac{\vm_t^{\theta}}{\sqrt{\vv_t^{\theta}} + \epsilon}.
\end{aligned}
\end{equation}
To simplify $\boldsymbol{\phi}$, we use Lemma \ref{lem:linear_norm}. Denoting $\vu =
(\|\mM_t^1\|_{\text{nuc}}, \ldots, \|\mM_t^L\|_{\text{nuc}},
\|\vm_t^{\theta}\|_{\text{ada}\infty,*})$, we have
\begin{equation}
    \boldsymbol{\phi} = -{\sf LMO}_f(\vu) = -\mD^{-1} {\sf LMO}_{\infty}(\mD^{-T} \vu) = \mD^{-1} \sign(\mD^{-T} \vu) = \mD^{-1} \vone,
\end{equation}
so that $\phi_{\ell} = 1$ for $\ell \in [L]$ and $\phi_{L+1} = \lambda =
\eta_b/\eta_m$. Plugging back to \Eqref{eq:muon_update_inter_3} gives
\begin{equation}
\begin{aligned}
    \mW_{t+1}^{\ell} &= \mW_t^{\ell} - \eta_m \text{polar}(\mM_t^{\ell}) \\
    \boldsymbol{\theta}_{t+1} &= \boldsymbol{\theta}_t - \eta_a \tfrac{\vm_t^{\theta}}{\sqrt{\vv_t^{\theta}} + \epsilon},
\end{aligned}
\end{equation}
which is exactly the update from Algorithm \ref{alg:muonadam}.
\end{proof}

\subsection{Recovering Existing Algorithms} \label{app:recover}
Propositions \ref{prop:scion} and \ref{prop:polargrad} below show how \alg{Scion}
\citep{scion} and \alg{PolarGrad} \citep{polargrad} are both instances of our steepest
descent framework. All notation in this section follows that of Section
\ref{sec:steepest}.

Throughout our paper, \alg{Scion} refers to the following algorithm:
\begin{equation}
\begin{aligned} \label{eq:scion}
    \mW_{t+1}^{\ell} &= \mW_t^{\ell} - \eta_m \text{polar}(\mM_t^{\ell}) \\
    \boldsymbol{\theta}_{t+1} &= \boldsymbol{\theta}_t - \eta_b \text{sign}(\vm_t^{\theta}).
\end{aligned}
\end{equation}
This update differs slightly from the algorithm proposed by \citet{scion} in that for
each parameter matrix $\mW$ of shape $d_{\text{out}} \times d_{\text{in}}$, we omit a
coefficient of $\sqrt{d_{\text{out}}/d_{\text{in}}}$ from the update. This corresponds
to assigning to each weight matrix the spectral norm $\|\cdot\|_{2 \rightarrow 2}$
rather than the RMS to RMS operator norm used by \citet{scion}. Indeed, the motivation
of the RMS to RMS norm is to allow for hyperparameter transfer across architecture
sizes, but in our work we focus on LR sensitivity for a fixed architecture, so for
simplicity we did not employ this RMS scaling. However, we could easily recover the RMS
variant by replacing the spectral norm $\|\cdot\|_{2 \rightarrow 2}$ with the RMS to RMS
operator norm.


\begin{proposition} \label{prop:scion}
\alg{Scion} is exactly CSD with step size $\eta_m$ with respect to
\begin{equation}
    \|\mW\|_{\text{scion}} = \max \left( \max_{1 \leq \ell \leq L} \|\mW^{\ell}\|_{2 \rightarrow 2}, \tfrac{\eta_m}{\eta_b} \|\boldsymbol{\theta}\|_{\infty} \right).
\end{equation}
\end{proposition}

Note that the same conclusion was already reached by \citet{scion}, that is, they
already described \alg{Scion} in terms of a norm on the space of all parameters (see
their Equation (6)). We include Proposition \ref{prop:scion} to specify how \alg{Scion}
is a special case of our framework.

\begin{proof}
The proof is very similar to that of Proposition \ref{prop:muon}, since
\alg{Muon}-\alg{Adam} differs from \alg{Scion} only in that \alg{Adam} is used for
non-matrix parameters instead of sign SGD with momentum.

By Proposition \ref{prop:csd}, one step of CSD w.r.t. $\|\cdot\|_{\text{scion}}$ can be
written as
\begin{equation} \label{eq:scion_update_inter}
    \mW_{t+1} = \mW_t + \eta_m {\sf LMO}_{\text{scion}}(\mM_t),
\end{equation}
where $\mM_t$ is the momentum buffer for all network parameters, i.e. it is the
concatenation of the momentum buffers of each parameter:
\begin{equation}
    \mM_t = (\mM_t^1, \ldots, \mM_t^L, \vm_t^{\theta}).
\end{equation}
Denote $\lambda = \eta_b / \eta_m$. To compute the ${\sf LMO}$ term, we can rewrite $\|\mW\|_{\text{scion}}$ as
\begin{equation}
    \|\mW\|_{\text{scion}} = \max \left( \|\mW^1\|_{2 \rightarrow 2}, \ldots \|\mW^L\|_{2 \rightarrow 2}, \tfrac{1}{\lambda} \|\boldsymbol{\theta}\|_{\infty} \right),
\end{equation}
so that $\|\cdot\|_{\text{scion}}$ can be written as the composition (as in Lemma \ref{lem:product_norm})
\begin{equation}
    \|\mW\|_{\text{scion}} = f(g_1(\mW^1), \ldots g_L(\mW^L), g_{L+1}(\boldsymbol{\theta})),
\end{equation}
with $g_i(\mM) = \|\mM\|_{2 \rightarrow 2}$ for $i \in [L]$,
$g_{L+1}(\boldsymbol{\theta}) = \|\boldsymbol{\theta}\|_{\infty}$, and $f(\vv) = \|\mD
\vv\|_{\infty}$, where $\mD = \text{Diag}(1, \ldots, 1, 1/\lambda) \in \R^{(L+1) \times
(L+1)}$. Therefore, by Lemma \ref{lem:product_norm}, the update in
\Eqref{eq:scion_update_inter} is equivalent to
\begin{equation} \label{eq:scion_update_inter_2}
\begin{aligned}
    \mW_{t+1}^{\ell} &= \mW_t^{\ell} + \eta_m \phi_{\ell} {\sf LMO}_{2 \rightarrow 2}(\mM_t^{\ell}) \\
    \boldsymbol{\theta}_{t+1} &= \boldsymbol{\theta}_t + \eta_m \phi_{L+1} {\sf LMO}_{\infty}(\vm_t^{\theta}),
\end{aligned}
\end{equation}
where $\boldsymbol{\phi} = -{\sf LMO}_f(\|\mM_t^1\|_{\text{nuc}}, \ldots,
\|\mM_t^L\|_{\text{nuc}}, \|\vm_t^{\theta}\|_{\text{ada}\infty,*})$. We know ${\sf
LMO}_{2 \rightarrow 2}(\mM) = -\text{polar}(\mM)$ and ${\sf LMO}_{\infty}(\vv) =
-\text{sign}(\vv)$, so the ${\sf LMO}$ terms in \Eqref{eq:scion_update_inter_2} can be
simplified as
\begin{equation}
\begin{aligned} \label{eq:scion_update_inter_3}
    \mW_{t+1}^{\ell} &= \mW_t^{\ell} - \eta_m \phi_{\ell} \text{polar}(\mM_t^{\ell}) \\
    \boldsymbol{\theta}_{t+1} &= \boldsymbol{\theta}_t - \eta_m \phi_{L+1} \text{sign}(\vm_t^{\theta}).
\end{aligned}
\end{equation}
To simplify $\boldsymbol{\phi}$, we use Lemma \ref{lem:linear_norm}. Denoting $\vu =
(\|\mM_t^1\|_{\text{nuc}}, \ldots, \|\mM_t^L\|_{\text{nuc}}, \|\vm_t^{\theta}\|_1)$, we
have
\begin{equation}
    \boldsymbol{\phi} = -{\sf LMO}_f(\vu) = -\mD^{-1} {\sf LMO}_{\infty}(\mD^{-T} \vu) = \mD^{-1} \sign(\mD^{-T} \vu) = \mD^{-1} \vone,
\end{equation}
so that $\phi_{\ell} = 1$ for $\ell \in [L]$ and $\phi_{L+1} = \lambda =
\eta_b/\eta_m$. Plugging back to \Eqref{eq:scion_update_inter_3} gives
\begin{equation}
\begin{aligned}
    \mW_{t+1}^{\ell} &= \mW_t^{\ell} - \eta_m \text{polar}(\mM_t^{\ell}) \\
    \boldsymbol{\theta}_{t+1} &= \boldsymbol{\theta}_t - \eta_a \text{sign}(\vm_t^{\theta}),
\end{aligned}
\end{equation}
which is exactly the update from \alg{Scion} (\Eqref{eq:scion}).
\end{proof}

Throughout our paper, \alg{PolarGrad} refers to the following algorithm:
\begin{equation} \label{eq:polargrad}
\begin{aligned}
    \mW_{t+1}^{\ell} &= \mW_t^{\ell} - \eta_s \|\mM_t^{\ell}\|_{\text{nuc}} \text{polar}(\mM_t^{\ell}) \\
    \boldsymbol{\theta}_{t+1} &= \boldsymbol{\theta}_t - \eta_b \frac{\vm_t^{\theta}}{\sqrt{\vv_t^{\theta}} + \epsilon}.
\end{aligned}
\end{equation}
\citet{polargrad} use the name "PolarGrad" to refer to a class of matrix-aware
optimization methods, whereas we use it to refer to the specific method called "Vanilla
PolarGrad" by \citet{polargrad} (see their Equation (8)), with Adam used for non-matrix
parameters.

\begin{proposition} \label{prop:polargrad}
\alg{PolarGrad} is exactly CSD with step size $\eta_m$ with respect to
\begin{equation}
    \|\mW\|_{\text{PG}} = \sqrt{\sum_{\ell=1}^L \|\mW^{\ell}\|_{2 \rightarrow 2}^2 + \frac{\eta_m}{\eta_b} \|\boldsymbol{\theta}\|_{\text{ada}2}^2}.
\end{equation}
\end{proposition}

\begin{proof}
Denote $\lambda = \eta_b / \eta_m$. Notice that $\|\cdot\|_{\text{PG}}$ can be written
as a composition (as in Lemma \ref{lem:product_norm}) as:
\begin{equation}
    \|\mW\|_{\text{PG}} = f(g_1(\mW^1), \ldots, g_L(\mW^L), g_{L+1}(\boldsymbol{\theta})),
\end{equation}
with $g_i(\mM) = \|\mM\|_{2 \rightarrow 2}$ for $i \leq L$,
$g_{L+1}(\boldsymbol{\theta}) = \|\boldsymbol{\theta}\|_{\text{ada}2}/\sqrt{\lambda}$,
and $f(\vv) = \|\vv\|_2$. Therefore, $\|\cdot\|_{\text{PG}}$ uses the $\ell_2$ norm as
the product norm, so \Eqref{eq:rsd_2} implies that the update can be rewritten as
\begin{equation}
\begin{aligned} \label{eq:pg_inter}
    \mW_{t+1}^{\ell} &= \mW_t^{\ell} + \eta_m \|\mM_t^{\ell}\|_{\text{nuc}} {\sf LMO}_{2 \rightarrow 2}(\mM_t^{\ell}) \\
    \boldsymbol{\theta}_{t+1} &= \boldsymbol{\theta}_t + \lambda \eta_m \|\vm_t^{\theta}\|_{\text{ada}2,*} {\sf LMO}_{\text{ada}2}(\vm_t^{\theta}).
\end{aligned}
\end{equation}
The update to $\mW_t^{\ell}$ can be simplified by plugging in ${\sf LMO}_{2 \rightarrow
2}(\mM) = -\text{polar}(\mM)$, and the update to $\boldsymbol{\theta}_t$ can be
simplified by plugging in the definition of $\lambda$ and the dual and {\sf LMO} of
$\|\cdot\|_{\text{ada}2}$ from Proposition \ref{prop:adam_2}. This yields that
\Eqref{eq:pg_inter} is equivalent to
\begin{equation}
\begin{aligned}
    \mW_{t+1}^{\ell} &= \mW_t^{\ell} - \eta_m \|\mM_t^{\ell}\|_{\text{nuc}} \text{polar}(\mM_t^{\ell}) \\
    \boldsymbol{\theta}_{t+1} &= \boldsymbol{\theta}_t - \eta_b \frac{\vm_t^{\theta}}{\sqrt{\vv_t^{\theta}} + \epsilon},
\end{aligned}
\end{equation}
which is exactly \alg{PolarGrad} (\Eqref{eq:polargrad}).
\end{proof}

\section{Proofs from Section \ref{sec:truncation}} \label{app:momo_proofs}

\propmomocsd*
\begin{proof}[Proof of Proposition \ref{prop:momo_csd}]
Similar to the proofs of Proposition \ref{prop:csd} and \ref{prop:rsd}, we change
variables to parameterize the magnitude $r = \|\vw - \vw_t\|$ and direction
$\boldsymbol{\Delta} = (\vw - \vw_t) / \|\vw - \vw_t\|$ of the update. So $\vw_{t+1} =
\vw_t + r_t \boldsymbol{\Delta_t}$, where
\begin{align}
    (r_t, \boldsymbol{\Delta}_t) = \argmin_{r \in [0, \eta], \|\boldsymbol{\Delta}\|=1} \left\{ \max \left( \tilde{F}_t + r \langle \vm_t, \boldsymbol{\Delta} \rangle, F_* \right) \right\}.
\end{align}
Since $\max \left( \tilde{F}_t + r \langle \vm_t, \boldsymbol{\Delta} \rangle, F_*
\right)$ is monotonic in $\langle \vm_t, \boldsymbol{\Delta} \rangle$,
\begin{align}
    \boldsymbol{\Delta}_t = \argmin_{\|\boldsymbol{\Delta}\|=1} \langle \vm_t, \boldsymbol{\Delta} \rangle = {\sf LMO}(\vm_t),
\end{align}
so
\begin{align}
    r_t = \argmin_{r \in [0, \eta]} \left\{ \max \left( \tilde{F}_t - r \langle \vm_t, \boldsymbol{\Delta_t} \rangle, F_* \right) \right\} = \argmin_{r \in [0, \eta]} \left\{ \max \left( \tilde{F}_t - r\|\vm_t\|_*, F_* \right) \right\}.
\end{align}
Note that $\max \left( \tilde{F}_t - r\|\vm_t\|_*, F_* \right)$ can have multiple
minimizing values of $r \in [0, \eta]$. If $\eta \leq (\tilde{F}_t - F_*)/\|\vm_t\|_*$,
then the minimizer $r = \eta$ is unique. If $\eta \geq (\tilde{F}_t - F_*)/\|\vm_t\|_*$,
then any $r$ with $(\tilde{F}_t - F_*)/\|\vm_t\|_* \leq r \leq \eta$ achieves the
minimum $F_*$. In this case, we choose the value that minimizes the norm of the update,
i.e. $r_t = (\tilde{F}_t - F_*)/\|\vm_t\|_*$. These two cases are summarized as:
\begin{equation}
    r_t = \min \left( \eta, \tfrac{\tilde{F}_t - F_*}{\|\vm_t\|_*} \right),
\end{equation}
so
\begin{equation}
    \vw_{t+1} = \vw_t + \min \left( \eta, \tfrac{\tilde{F}_t - F_*}{\|\vm_t\|_*} \right) {\sf LMO}(\vm_t).
\end{equation}
\end{proof}

\propmomorsd*
\begin{proof}[Proof of Proposition \ref{prop:momo_rsd}]
Similar to the proofs of Proposition \ref{prop:csd} and \ref{prop:rsd}, we perform a
change of variables to parameterize the magnitude $r = \|\vw - \vw_t\|$ and direction
$\boldsymbol{\Delta} = (\vw - \vw_t) / \|\vw - \vw_t\|$ of the update. So $\vw_{t+1} =
\vw_t + r_t \boldsymbol{\Delta}_t$, where
\begin{equation}
    (r_t, \boldsymbol{\Delta}_t) = \argmin_{r \geq 0, \|\boldsymbol{\Delta}\| = 1} \left\{ \max \left( \tilde{F}_t + r \langle \vm_t, \boldsymbol{\Delta} \rangle, F_* \right) + \tfrac{r^2}{2 \eta} \right\}.
\end{equation}
Note that $\max \left( \tilde{F}_t + r \langle \vm_t, \boldsymbol{\Delta} \rangle, F_*
\right) + \tfrac{r^2}{2 \eta}$ is monotonic in $\langle \vm_t, \boldsymbol{\Delta}
\rangle$, so
\begin{align}
    \boldsymbol{\Delta}_t = \argmin_{\|\boldsymbol{\Delta}\|=1} \left\{ \langle \vm_t, \boldsymbol{\Delta} \rangle \right\} = {\sf LMO}(\vm_t),
\end{align}
and
\begin{align}
    r_t &= \argmin_{r \geq 0} \left\{ \max \left( \tilde{F}_t + r \langle \vm_t, \boldsymbol{\Delta}_t \rangle, F_* \right) + \tfrac{r^2}{2 \eta} \right\} \\
    &= \argmin_{r \geq 0} \left\{ \max \left( \tilde{F}_t - r \|\vm_t\|_*, F_* \right) + \tfrac{r^2}{2 \eta} \right\}.
\end{align}
Denote $f(r) = \max \left( \tilde{F}_t - r \|\vm_t\|_*, F_* \right) + \tfrac{r^2}{2
\eta}$. Then $f$ can be written piecewise as
\begin{equation}
    f(r) = \begin{cases}
        \tilde{F}_t - r \|\vm_t\|_* + \tfrac{r^2}{2 \eta} & r \leq \tfrac{\tilde{F}_t - F_*}{\|\vm_t\|_*} \\
        F_* + \tfrac{r^2}{2 \eta} & r \geq \tfrac{\tilde{F}_t - F_*}{\|\vm_t\|_*}
    \end{cases}.
\end{equation}
Note that $f$ is increasing for $r \geq (\tilde{F}_t - F_*)/\|\vm_t\|_*$, so its
minimizer is the minimizer of $\tilde{F}_t - r \|\vm_t\|_* + \tfrac{r^2}{2 \eta}$ for $r
\leq (\tilde{F}_t - F_*)/\|\vm_t\|_*$. So
\begin{equation}
    r_t = \min \left( \eta \|\vm_t\|_*, \tfrac{\tilde{F}_t - F_*}{\|\vm_t\|_*} \right),
\end{equation}
therefore
\begin{equation}
    \vw_{t+1} = \vw_t + \left( \eta, \tfrac{\tilde{F}_t - F_*}{\|\vm_t\|_*^2} \right) \|\vm_t\|_* {\sf LMO}(\vm_t).
\end{equation}
Note that this value of $\vw_{t+1}$ is the unique minimizer of
\begin{equation}
    \max \left( \tilde{F}_t + \langle \vm_t, \vw - \vw_t \rangle, F_* \right) + \tfrac{1}{2 \eta} \|\vw - \vw_t\|^2,
\end{equation}
since this is function is strongly convex (sum of a convex function and a strongly
convex function), and therefore has a unique minimizer.
\end{proof}

The pseudocode for Constrained Momo and Regularized Momo are shown in Algorithm
\ref{alg:momo}. To see why this algorithm correctly computes $\tilde{F}_t$, note that
\begin{align}
    \tilde{F}_t &= \sum_{i=0}^t \rho_{t,i} \left( F_i(\vw_i) + \langle \vg_i, \vw_t - \vw_i \rangle \right) \\
    &= \sum_{i=0}^t \rho_{t,i} \left( F_i(\vw_i) - \langle \vg_i, \vw_i \rangle \right) + \sum_{i=0}^t \rho_{t,i} \langle \vg_i, \vw_t \rangle \\
    &= \sum_{i=0}^t \rho_{t,i} \left( F_i(\vw_i) - \langle \vg_i, \vw_i \rangle \right) + \langle \vm_t, \vw_t \rangle.
\end{align}
So denoting $\tilde{f}_t = \sum_{i=0}^t \rho_{t,i} \left( F_i(\vw_i) - \langle \vg_i,
\vw_i \rangle \right)$, we have $\tilde{F}_t = \tilde{f}_t + \langle \vm_t, \vw_t
\rangle$, and
\begin{equation}
    \tilde{f}_t = \beta \tilde{f}_{t-1} + (1-\beta) \left( F_t(\vw_t) - \langle \vg_t, \vw_t \rangle \right),
\end{equation}
so that $\tilde{f}_t$ is given by the running average used in Algorithm \ref{alg:momo}.

\begin{algorithm}[t]
    \caption{Momo (Constrained or Regularized)}
    \label{alg:momo}
    \AlgIO{Inputs}{$\vw_0$, learning rate $\eta$, momentum $\beta$, loss lower bound $F_*$}
    \begin{algorithmic}[1]
        \For{$t = 0, 1, \ldots, T-1$}
            \State $\vg_t \gets \text{backward}(\vw_t)$
            \State $\vm_t = \beta \vm_{t-1} + (1 - \beta) \vg_t$
            \State $\tilde{f}_t = \beta \tilde{f}_{t-1} + (1 - \beta) \left( F_t(\vw_t) - \langle \vg_t, \vw_t \rangle \right)$
            \State $\tilde{F}_t = \tilde{f}_t + \langle \vm_t, \vw_t \rangle$
            \If{\text{Constrained}}
                \State $\vw_{t+1} = \vw_t + \min \left( \eta, \tfrac{\tilde{F}_t - F_*}{\|\vm_t\|_*} \right) {\sf LMO}(\vm_t)$
            \Else{}
                \State $\vw_{t+1} = \vw_t + \min \left( \eta, \tfrac{\tilde{F}_t - F_*}{\|\vm_t\|_*^2} \right) \|\vm_t\|_* {\sf LMO}(\vm_t)$
            \EndIf
        \EndFor
    \end{algorithmic}
\end{algorithm}

Now we derive the closed-form update for our proposed algorithm
\alg{MuonMax}-\alg{Momo}. Algorithm \ref{alg:muonmax} has the pseudocode for the
algorithm, and Proposition \ref{prop:muonmax} proves that this procedure implements
Regularized \alg{Momo} with respect to $\|\cdot\|_{\text{MM}}$. Note that Algorithm
\ref{alg:muonmax} shows the pseudocode with stale nuclear norm approximations, while
Proposition \ref{prop:muonmax} considers the vanilla version.

It should be noted that, if we set $\beta = 0$, the stepsize scaling $\sum_{\ell=1}^L
\|\mG_t^{\ell}\|_{\text{nuc}}$ for the matrix layers in Algorithm \ref{alg:muonmax} was
previously mentioned by \citet{anthology} (see their Proposition 5). However, we are not
aware of any existing implementation or evaluation of this stepsize scaling, and we
found in our experiments that this sort of scaling (without model truncation) is not
competitive with \alg{Muon}.

\begin{algorithm}[t]
    \caption{MuonMax-Momo}
    \label{alg:muonmax}
    \AlgIO{Inputs}{$\mW_0^1, \ldots, \mW_0^L, \boldsymbol{\theta}_0$, learning rates $\eta_m, \eta_b$, EMA parameters $\beta, \beta_2$, loss lower bound $F_*$}
    \AlgIO{Defaults}{$\eta_m = \eta_b = 0.01$, $\beta = \beta_2 = 0.95$}
    \begin{algorithmic}[1]
        \For{$t = 0, 1, \ldots, T-1$}
            \item[]
            \item[]
            \State \tikzmark{mom-beg} $(\mG_t^1, \ldots, \mG_t^L, \vg_t^{\theta}) \gets \text{backward}(\mW_t^1, \ldots \mW_t^L, \boldsymbol{\theta}_t)$
            \For{$\ell = 1, \ldots, L$}
                \State $\mM_t^{\ell} = \beta \mM_{t-1}^{\ell} + (1 - \beta) \mG_t^{\ell}$
            \EndFor
            \State $\vm_t^{\theta} = \beta \vm_{t-1}^{\theta} + (1 - \beta) \vg_t^{\theta}$
            \State $\vv_t^{\theta} = \beta_2 \vv_{t-1}^{\theta} + (1 - \beta_2) \vg_t^{\theta} \odot \vg_t^{\theta}$ \hspace{2.5cm}\tikzmark{mom-end}
            
            \item[]
            \item[]
            \State \tikzmark{trunc-beg} $\tilde{f}_t = \beta \tilde{f}_{t-1} + (1 - \beta) \left( F_t(\mW_t) - \sum_{\ell=1}^L \langle \mG_t^{\ell}, \mW_t^{\ell} \rangle - \langle \vg_t^{\theta}, \vm_t^{\theta} \right)$
            \State $\tilde{F}_t = \tilde{f}_t + \sum_{\ell=1}^L \langle \mM_t^{\ell}, \mW_t^{\ell} \rangle + \langle \vm_t^{\theta}, \boldsymbol{\theta}_t \rangle$
            \State $d_t = \sqrt{ \left( \sum_{\ell=1}^L d_{t-1}^{\ell} \right)^2 + \frac{\eta_b}{\eta_m} \big\| \tfrac{\vm_t^{\theta}}{\sqrt{ \vv_t^{\theta} + \epsilon}} \big\|_2^2 }$ \hspace{3.25cm}\tikzmark{trunc-end}

            \item[]
            \item[]
            \For{$\ell = 1, \ldots, L$} \hspace{-3cm}\tikzmark{update-beg} 
                \State $\mP \gets \text{polar}(\mM_t^{\ell})$
                \State $\mW_{t+1}^{\ell} \gets \mW_t^{\ell} - \min \left( \eta_m, \tfrac{\tilde{F}_t - F_*}{d_t^2} \right) \left( \sum_{j=1}^L d_{t-1}^{\ell} \right) \mP$
                \State $d_t^{\ell} \gets \langle \mP, \mM_t^{\ell} \rangle$
            \EndFor
            \State $\boldsymbol{\theta}_{t+1} = \boldsymbol{\theta}_t - \min \left( \eta_b, \frac{\eta_b}{\eta_m} \tfrac{\tilde{F}_t - F_*}{d_t^2} \right) \tfrac{\vm_t^{\theta}}{\sqrt{\vv_t^{\theta}} + \epsilon}$ \hspace{3cm}\tikzmark{update-end}
        \item[]
        \EndFor
    \end{algorithmic}
{
  \setlength{\AlgBoxPadTop}{1.2ex}
  \setlength{\AlgBoxPadBottom}{0.2ex}
  \AlgBox{mom-beg}{mom-end}{algC1}{Update momentum.}
}
{
  \setlength{\AlgBoxPadTop}{1.4ex}
  \setlength{\AlgBoxPadBottom}{0.9ex}
  \AlgBox{trunc-beg}{trunc-end}{algC2}{Update internal truncation variables.}
}
{
  \setlength{\AlgBoxPadTop}{0.8ex}
  \setlength{\AlgBoxPadBottom}{0.8ex}
  \AlgBox{update-beg}{update-end}{algC3}{Update parameters.}
}
\end{algorithm}

\propmuonmax*
\begin{proof}[Proof of Proposition \ref{prop:muonmax}]
The proof structure is largely similar to that of Proposition \ref{prop:muon}. By
Proposition \ref{prop:momo_rsd}, one step of Regularized Momo w.r.t.
$\|\cdot\|_{\text{MM}}$ can be written as
\begin{equation} \label{eq:muonmax_update_inter}
    \mW_{t+1} = \mW_t + \min \left( \eta_m, \tfrac{\tilde{F}_t - F_*}{\|\mM_t\|_{\text{MM},*}^2} \right) \|\mM_t\|_{\text{MM},*} {\sf LMO}_{\text{MM}}(\mM_t),
\end{equation}
where $\mM_t$ is the momentum buffer for all network parameters, i.e. it is the
concatenation of the momentum buffers of each parameter:
\begin{equation}
    \mM_t = (\mM_t^1, \ldots, \mM_t^L, \vm_t^{\theta}).
\end{equation}
Comparing \Eqref{eq:muonmax_update_inter} with \Eqref{eq:muonmax_momo}, we have to prove
that $d_t = \|\mM_t\|_{\text{MM},*}$ and compute $\|\mM_t\|_{\text{MM},*} {\sf
LMO}_{\text{MM}}(\mM_t)$. To do this, we write $\|\cdot\|_{\text{MM}}$ with
repeated compositions of norms whose dual and {\sf LMO} we already know.
Denoting $\lambda = \eta_b/\eta_m$ and
\begin{align}
    f(z_1, z_2) &= \sqrt{z_1^2 + \tfrac{1}{\lambda} z_2^2} \\
    g_1(\mW_1, \ldots, \mW_L) &= \max_{\ell \in [L]} \|\mW_{\ell}\|_{2 \rightarrow 2} \\
    g_2(\boldsymbol{\theta}) &= \|\boldsymbol{\theta}\|_{\text{ada}2},
\end{align}
we can write $\|\cdot\|_{\text{MM}}$ as a composition in the notation of Lemma
\ref{lem:product_norm} as
\begin{equation}
    \|\mW\|_{\text{MM}} = f(g_1(\mW_1, \ldots, \mW_L), g_2(\boldsymbol{\theta})).
\end{equation}
Further denoting $\mD = \text{diag}(1, 1/\sqrt{\lambda})$, we can write $f(z_1, z_2) =
\|\mD (z_1, z_2)^T\|_2$. We can now use Lemma \ref{lem:product_norm} to compute the
dual of $\|\cdot\|_{\text{MM},*}$ as
\begin{align}
    \|\mW\|_{\text{MM},*} &= f_*(g_{1,*}(\mW_1, \ldots, \mW_L), g_{2,*}(\boldsymbol{\theta})) \\
    &\Eqmark{i}{=} \sqrt{ g_{1,*}^2(\mW_1, \ldots, \mW_L) + \lambda g_{2,*}^2(\boldsymbol{\theta}) } \\
    &\Eqmark{ii}{=} \sqrt{ g_{1,*}^2(\mW_1, \ldots, \mW_L) + \lambda \left\| \tfrac{\boldsymbol{\theta}}{\sqrt{\sqrt{\vv_t^{\theta} + \epsilon}}} \right\|^2 } \\
    &\Eqmark{iii}{=} \sqrt{ \left( \sum_{\ell=1}^L \|\mW_{\ell}\|_{\text{nuc}} \right)^2 + \lambda \left\| \tfrac{\boldsymbol{\theta}}{\sqrt{\sqrt{\vv_t^{\theta} + \epsilon}}} \right\|^2 }
\end{align}
where $(i)$ uses Lemma \ref{lem:linear_norm} to plug in the dual of $f$, $(ii)$ plugs in
the dual of $\|\cdot\|_{\text{ada}2}$ which we computed in the proof of Proposition
\ref{prop:adam_2}, and $(iii)$ uses Lemma \ref{lem:product_norm} to compute the dual of
$g_1$, which itself is a composition $g_1 = \ell_{\infty} \circ (\|\cdot\|_{2
\rightarrow 2}, \ldots, \|\cdot\|_{2 \rightarrow 2})$. This confirms that $d_t =
\|\mM_t\|_{\text{MM},*}$, so
\begin{equation} \label{eq:MM_update_inter}
    \mW_{t+1} = \mW_t + \min \left( \eta_m, \tfrac{\tilde{F}_t - F_*}{d_t^2} \right) d_t {\sf LMO}_{\text{MM}}(\mM_t),
\end{equation}
To compute the {\sf LMO} of $\|\cdot\|_{\text{MM}}$, we again use Lemma
\ref{lem:product_norm}. Denoting $(\phi_1, \phi_2) = -{\sf LMO}_f(g_{1,*}(\mW_1, \ldots,
\mW_L), g_{2,*}(\boldsymbol{\theta}))$, Lemma \ref{lem:product_norm} implies
\begin{align}
    {\sf LMO}_{\text{MM}}(\mW) &= (\phi_1 {\sf LMO}_{g_1}(\mW_1, \ldots, \mW_L), \phi_2 {\sf LMO}_{g_2}(\boldsymbol{\theta})) \\
    &\Eqmark{i}{=} (-\phi_1 (\text{polar}(\mW_1), \ldots, \text{polar}(\mW_L)), \phi_2 {\sf LMO}_{g_2}(\boldsymbol{\theta}))  \\
    &\Eqmark{ii}{=} -\left( \phi_1 (\text{polar}(\mW_1), \ldots, \text{polar}(\mW_L)), \phi_2 \tfrac{\boldsymbol{\theta}}{\sqrt{\vv_t}+\epsilon} \bigg/ \left\| \tfrac{\boldsymbol{\theta}}{\sqrt{\sqrt{\vv_t}+\epsilon}} \right\|_2 \right), \label{eq:MM_lmo_inter}
\end{align}
where $(i)$ uses Lemma \ref{lem:product_norm} to compute the LMO of $g_1$, which again
is the composition $g_1 = \ell_{\infty} \circ (\|\cdot\|_{2 \rightarrow 2}, \ldots,
\|\cdot\|_{2 \rightarrow 2})$, and $(iii)$ uses Lemma \ref{lem:linear_norm} to plug in
the dual norm of $g_2 = \|\cdot\|_{\text{ada}2}$. The $\phi$ terms can be simplified as
\begin{align}
    (\phi_1, \phi_2) &= -{\sf LMO}_f(g_{1,*}(\mW_1, \ldots, \mW_L), g_{2,*}(\boldsymbol{\theta})) \\
    &\Eqmark{i}{=} -{\sf LMO}_f \left( \sum_{\ell=1}^L \|\mW_{\ell}\|_{\text{nuc}}, \left\| \tfrac{\boldsymbol{\theta}}{\sqrt{\sqrt{\vv_t} + \epsilon}} \right\|_2 \right) \\
    &\Eqmark{ii}{=} -\mD^{-1} {\sf LMO}_2 \left( \sum_{\ell=1}^L \|\mW_{\ell}\|_{\text{nuc}}, \sqrt{\lambda} \left\| \tfrac{\boldsymbol{\theta}}{\sqrt{\sqrt{\vv_t} + \epsilon}} \right\|_2 \right) \\
    &= \tfrac{1}{d_t} \mD^{-1} \left( \sum_{\ell=1}^L \|\mW_{\ell}\|_{\text{nuc}}, \sqrt{\lambda} \left\| \tfrac{\boldsymbol{\theta}}{\sqrt{\sqrt{\vv_t} + \epsilon}} \right\|_2 \right) \\
    &= \tfrac{1}{d_t} \left( \sum_{\ell=1}^L \|\mW_{\ell}\|_{\text{nuc}}, \lambda \left\| \tfrac{\boldsymbol{\theta}}{\sqrt{\sqrt{\vv_t} + \epsilon}} \right\|_2 \right),
\end{align}
where $(i)$ plugs in the previously computed duals $g_{1,*}$ and $g_{2,*}$, and $(ii)$
uses Lemma \ref{lem:linear_norm} to plug in the {\sf LMO} of $f$. Plugging the values of
$(\phi_1, \phi_2)$ into \Eqref{eq:MM_lmo_inter} yields
\begin{equation}
    {\sf LMO}_{\text{MM}}(\mW) = -\tfrac{1}{d_t} \left( \left( \sum_{\ell=1}^L \|\mW_{\ell}\|_{\text{nuc}} \right) (\text{polar}(\mW_1), \ldots, \text{polar}(\mW_L)), \lambda \tfrac{\boldsymbol{\theta}}{\sqrt{\vv_t} + \epsilon} \right),
\end{equation}
and finally, plugging this back into \Eqref{eq:MM_update_inter} yields
\begin{align}
    \mW_{t+1}^{\ell} &= \mW_t - \min \left( \eta_m, \tfrac{\tilde{F}_t - F_*}{d_t^2} \right) \left( \sum_{i=1}^L \|\mW_i\|_{\text{nuc}} \right) \text{polar}(\mM_t^{\ell}) \\
    \boldsymbol{\theta}_{t+1} &= \boldsymbol{\theta}_t - \min \left( \eta_m, \tfrac{\tilde{F}_t - F_*}{d_t^2} \right) \lambda \tfrac{\vm_t^{\theta}}{\sqrt{\vv_t^{\theta}} + \epsilon} = \boldsymbol{\theta}_t - \min \left( \eta_b, \tfrac{\eta_b}{\eta_m} \tfrac{\tilde{F}_t - F_*}{d_t^2} \right) \tfrac{\vm_t^{\theta}}{\sqrt{\vv_t^{\theta}} + \epsilon},
\end{align}
which is exactly the update in \Eqref{eq:muonmax_momo}.
\end{proof}

\section{Experimental Details} \label{app:experiment_details}

\paragraph{Setup} We did not use weight decay or Nesterov momentum, as we found both to
have very small effects on final loss. All methods use a warmup-stable-decay learning
rate schedule, where the learning rate is linearly warmed up for the first 5\% of steps,
held constant until halfway through training, then linearly decayed to 10\% of the
warmed up value. We use a context length of 1024 and a batch size of 512. Rather than
the Newton-Schulz iterations of the original \alg{Muon} implementation, we use the
PolarExpress algorithm \citep{amsel2025polar} to compute approximate polar factors. In
this implementation, the weights and gradients are computed in \texttt{float32}, whereas
the polar factor is computed in \texttt{bfloat16} by the
PolarExpress~\citep{amsel2025polar}.

\paragraph{Tuning Protocol}
For the experiments with FineWeb data in Section \ref{sec:fineweb}, we tune 36
variations of steepest descent using an iterated grid search to for the two learning
rates $\eta_m$ and $\eta_b$. For the 18 variations without model truncation, we first
fix the base learning rate at an intermediate value $\eta_b = $1e-3, then tune the
\alg{Muon} learning rate with grid search over $\eta_m \in \{$1e-3, 1e-2, 1e-1, 1$\}$.
Some algorithms diverged with $\eta_b = $1e-3, and for these algorithms we instead used
$\eta_b = $1e-6 and searched over $\eta_m \in \{$1e-6, 1e-5, 1e-4, 1e-3$\}$. For those
algorithms that used $\eta_b = $1e-6 for the first phase, we instead search over $\eta_b
\in \{$1e-7, 1e-6, 1e-5, 1e-4$\}$ in the second phase. Finally, for all of these grid
searches, we extend the search space individually for each algorithm until the best LR
is not a boundary point of the search space. The resulting tuned LRs are shown in Table
\ref{tab:all_fineweb_losses}.

For the 18 variations with model truncation, rather than entirely retuning all
algorithms, we reuse the tuned LR ratio $\eta_m/\eta_b$ and do a single grid search
where $\eta_m$ and $\eta_b$ scale together. More specifically, we run each algorithm
with LRs $(\rho \eta_m, \rho \eta_b)$, where $(\eta_m, \eta_b)$ are the LRs tuned for
each algorithm without truncation, and the scaling factor $\rho$ ranges over $\rho \in
\{0.3, 1, 3, 10, 30, 100\}$. We found that the best value of $\rho$ for each algorithm
was always at least $1$ and at most $30$. The resulting tuned LRs are shown in Table
\ref{tab:all_fineweb_losses_truncated}.

\paragraph{Hybrid Norm Definition} Recall that \alg{Muon}-\alg{Max} is defined
as regularized steepest descent with respect to the following norm:
\begin{equation}
    \|\mW\|_{\text{MM}} = \textstyle\sqrt{ \big( \max_{\ell \in [L]} \|\mW^{\ell}\|_{2 \rightarrow 2} \big)^2 + \tfrac{\eta_m}{\eta_b} \|\boldsymbol{\theta}\|_{\text{ada2}}^2 }.
\end{equation}
This norm fits into our framework by assigning the spectral norm to each weight
matrix $\mW_{\ell}$, assigning $\|\cdot\|_{\text{ada2}}$ to the remaining
parameters, and aggregating norms for all parameters with the following "hybrid"
product norm:
\begin{equation} \label{eq:hyb_norm_def}
    \|(v_1, \ldots, v_L, v_{L+1})\|_{\text{hyb}} = \sqrt{ \big( \max_{\ell \in [L]} v_{\ell} \big)^2 + \tfrac{\eta_m}{\eta_b} v_{L+1}^2 }.
\end{equation}

\section{Additional Experimental Results} \label{app:experiment_results}

\subsection{FineWeb}
The final validation loss reached by all 36 of our evaluated methods is shown in Tables
\ref{tab:all_fineweb_losses} and \ref{tab:all_fineweb_losses_truncated}. Each method is
denoted as a 3-tuple of settings from our steepest descent framework: regularized vs
constrained steepest descent, product norm, and norm for parameters besides hidden
weight matrices.

For the methods without model truncation (Table \ref{tab:all_fineweb_losses}), we see
that the RSD methods struggle generally lag behind the CSD methods, likely due to a lack
of update normalization. For the CSD methods, \alg{Muon} and \alg{Scion} are among the
best variations, though the best performing method is actually (Constrained,
$\|\cdot\|_{\text{hyb}}$, $\|\cdot\|_{\text{ada}2}$) (we will return to discuss this
method shortly).

For the methods with model truncation (Table \ref{tab:all_fineweb_losses_truncated}), we
see that both CSD and RSD methods are competitive, meaning that in general model
truncation helped RSD methods more than CSD methods (at least in terms of final loss
with tuned LRs). \alg{Muon}-\alg{Momo} has the lowest loss at 3.551 and
\alg{Scion}-\alg{Momo} is again among the best performers, but actually many methods
achieve losses very close to 3.551. Again, we see that (Constrained,
$\|\cdot\|_{\text{hyb}}$, $\|\cdot\|_{\text{ada}2}$) achieves a very low loss, only
being outperformed by \alg{Muon}-\alg{Momo}.

The method (Constrained, $\|\cdot\|_{\text{hyb}}$, $\|\cdot\|_{\text{ada}2}$) is quite
similar to our proposed method \alg{Muon}-\alg{Max}, the only difference being the use
of a normalized upate. While this method does achieve a lower loss after tuning than
\alg{MuonMax}, we found that this method was not as robust to learning rate tuning. So
this method was bested by \alg{Muon}-\alg{Momo} in terms of final loss, and it was
bested by \alg{MuonMax}-\alg{Momo} in terms of learning rate sensitivity, and for this
reason we did not perform further evaluations with this method.

\begin{table}[t]
\centering
\caption{Final validation losses for all variations without model truncation.}
\begin{tabular}{lcccc}
\toprule
(SD type, Product Norm, Backup Norm) & Muon LR & Other LR & Final Loss & Name \\
\midrule
(Regularized, $\|\cdot\|_{\infty}$, $\|\cdot\|_{\infty}$) & 1e-3 & 1e-5 & 3.783 & - \\
(Constrained, $\|\cdot\|_{\infty}$, $\|\cdot\|_{\infty}$) & 1e-2 & 1e-3 & 3.599 & Scion \\
(Regularized, $\|\cdot\|_2$, $\|\cdot\|_{\infty}$) & 1e-1 & 1e-6 & 4.179 & - \\
(Constrained, $\|\cdot\|_2$, $\|\cdot\|_{\infty}$) & 1e-1 & 1e-2 & 3.712 & - \\
(Regularized, $\|\cdot\|_{\text{hyb}}$, $\|\cdot\|_{\infty}$) & 1e-3 & 1e-5 & 3.826 & - \\
(Constrained, $\|\cdot\|_{\text{hyb}}$, $\|\cdot\|_{\infty}$) & 1e-2 & 1e-3 & 3.610 & - \\
(Regularized, $\|\cdot\|_{\infty}$, $\|\cdot\|_{\text{ada}\infty}$) & 1e-3 & 1e-5 & 3.859 & - \\
(Constrained, $\|\cdot\|_{\infty}$, $\|\cdot\|_{\text{ada}\infty}$) & 1e-2 & 1e-3 & 3.604 & Muon \\
(Regularized, $\|\cdot\|_2$, $\|\cdot\|_{\text{ada}\infty}$) & 1e-1 & 1e-4 & 4.229 & - \\
(Constrained, $\|\cdot\|_2$, $\|\cdot\|_{\text{ada}\infty}$) & 1e-1 & 1e-2 & 3.748 & - \\
(Regularized, $\|\cdot\|_{\text{hyb}}$, $\|\cdot\|_{\text{ada}\infty}$) & 1e-3 & 1e-4 & 3.917 & - \\
(Constrained, $\|\cdot\|_{\text{hyb}}$, $\|\cdot\|_{\text{ada}\infty}$) & 1e-2 & 1e-2 & 3.628 & - \\
(Regularized, $\|\cdot\|_{\infty}$, $\|\cdot\|_{\text{ada}2}$) & 1e-3 & 1e-4 & 3.757 & - \\
(Constrained, $\|\cdot\|_{\infty}$, $\|\cdot\|_{\text{ada}2}$) & 1e-2 & 1e-3 & 3.701 & - \\
(Regularized, $\|\cdot\|_2$, $\|\cdot\|_{\text{ada}2}$) & 1e-1 & 1e-3 & 4.049 & PolarGrad \\
(Constrained, $\|\cdot\|_2$, $\|\cdot\|_{\text{ada}2}$) & 1e-1 & 1e-2 & 3.664 & - \\
(Regularized, $\|\cdot\|_{\text{hyb}}$, $\|\cdot\|_{\text{ada}2}$) & 1e-3 & 1e-3 & 3.791 & MuonMax \\
(Constrained, $\|\cdot\|_{\text{hyb}}$, $\|\cdot\|_{\text{ada}2}$) & 1e-2 & 1e-2 & 3.585 & - \\
\bottomrule
\end{tabular}
\label{tab:all_fineweb_losses} 
\end{table}

\begin{table}[t]
\centering
\caption{Final validation losses for all variations with model truncation.}
\begin{tabular}{lcccc}
\toprule
(SD type, Product Norm, Backup Norm) & Muon LR & Other LR & Final Loss & Name \\
\midrule
(Regularized, $\|\cdot\|_{\infty}$, $\|\cdot\|_{\infty}$) & 1e-2 & 1e-4 & 3.627 & - \\
(Constrained, $\|\cdot\|_{\infty}$, $\|\cdot\|_{\infty}$) & 1e-2 & 1e-3 & 3.592 & Scion-Momo \\
(Regularized, $\|\cdot\|_2$, $\|\cdot\|_{\infty}$) & 1 & 1e-5 & 3.728 & - \\
(Constrained, $\|\cdot\|_2$, $\|\cdot\|_{\infty}$) & 1e-1 & 1e-2 & 3.843 & - \\
(Regularized, $\|\cdot\|_{\text{hyb}}$, $\|\cdot\|_{\infty}$) & 1e-2 & 1e-4 & 3.628 & - \\
(Constrained, $\|\cdot\|_{\text{hyb}}$, $\|\cdot\|_{\infty}$) & 3e-2 & 3e-3 & 3.604 & - \\
(Regularized, $\|\cdot\|_{\infty}$, $\|\cdot\|_{\text{ada}\infty}$) & 3e-2 & 3e-4 & 3.578 & - \\
(Constrained, $\|\cdot\|_{\infty}$, $\|\cdot\|_{\text{ada}\infty}$) & 3e-2 & 3e-3 & 3.551 & Muon-Momo \\
(Regularized, $\|\cdot\|_2$, $\|\cdot\|_{\text{ada}\infty}$) & 1 & 1e-3 & 3.719 & - \\
(Constrained, $\|\cdot\|_2$, $\|\cdot\|_{\text{ada}\infty}$) & 1e-1 & 1e-2 & 3.737 & - \\
(Regularized, $\|\cdot\|_{\text{hyb}}$, $\|\cdot\|_{\text{ada}\infty}$) & 3e-2 & 3e-3 & 3.584 & - \\
(Constrained, $\|\cdot\|_{\text{hyb}}$, $\|\cdot\|_{\text{ada}\infty}$) & 3e-2 & 3e-2 & 3.607 & - \\
(Regularized, $\|\cdot\|_{\infty}$, $\|\cdot\|_{\text{ada}2}$) & 3e-3 & 3e-4 & 3.662 & - \\
(Constrained, $\|\cdot\|_{\infty}$, $\|\cdot\|_{\text{ada}2}$) & 1e-2 & 1e-3 & 3.701 & - \\
(Regularized, $\|\cdot\|_2$, $\|\cdot\|_{\text{ada}2}$) & 3 & 3e-2 & 3.613 & PolarGrad-Momo \\
(Constrained, $\|\cdot\|_2$, $\|\cdot\|_{\text{ada}2}$) & 3e-1 & 3e-2 & 3.602 & - \\
(Regularized, $\|\cdot\|_{\text{hyb}}$, $\|\cdot\|_{\text{ada}2}$) & 1e-2 & 1e-2 & 3.576 & MuonMax-Momo \\
(Constrained, $\|\cdot\|_{\text{hyb}}$, $\|\cdot\|_{\text{ada}2}$) & 3e-2 & 3e-2 & 3.553 & - \\
\bottomrule
\end{tabular}
\label{tab:all_fineweb_losses_truncated} 
\end{table}

We also include loss curves for the last 40\% of training for the best variations (with
tuned learning rates) in Figure \ref{fig:fineweb1B_loss_curves}, and the final losses
reached by the best variations (over three seeds) in Table
\ref{tab:final_fineweb_losses}. Lastly, Figure \ref{fig:fineweb1B_trunc_comparison}
shows a comparison of \alg{MuonAdam}, \alg{Scion}, \alg{MuonMax} against their truncated
counterparts.

\begin{figure}[t]
\centering
\begin{subfigure}[b]{0.49\textwidth}
\centering    
\includegraphics[width=\textwidth]{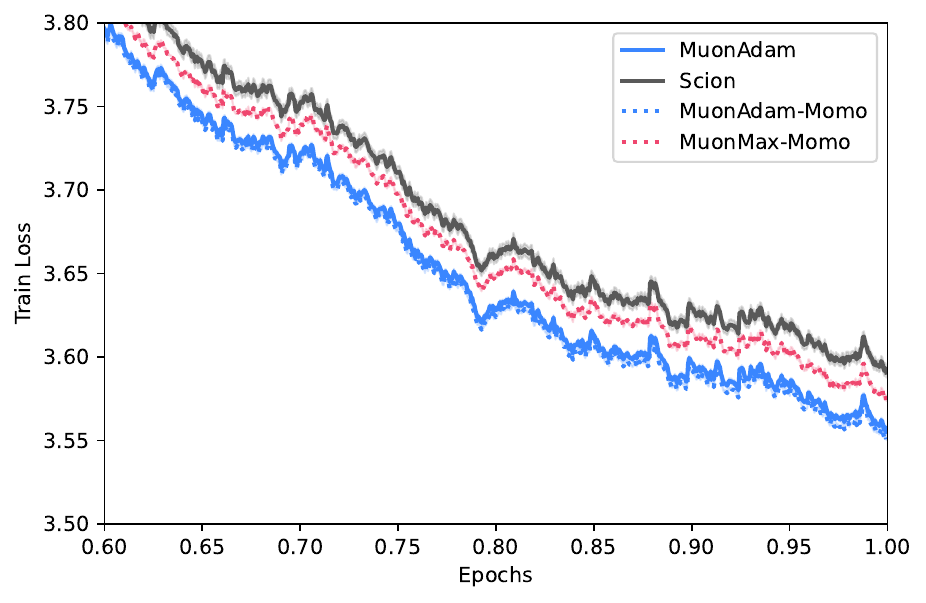}
\caption{FineWeb1B.}
\label{fig:fineweb1B_loss_curves}
\end{subfigure}
\begin{subfigure}[b]{0.49\textwidth}
\centering
\includegraphics[width=\textwidth]{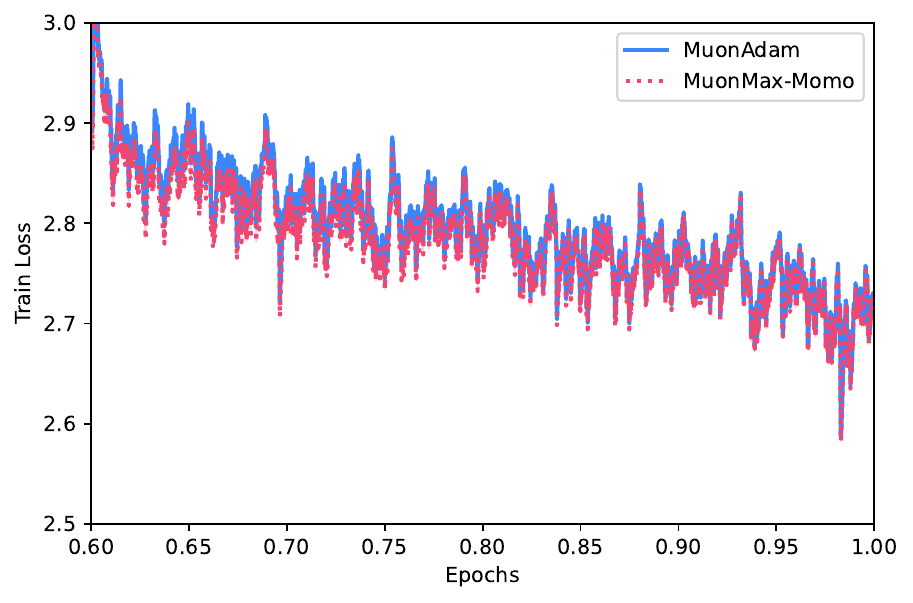}
\caption{SlimPajama6B.}
\label{fig:slimpajama6B_loss_curves}
\end{subfigure}
\caption{Training loss for the last 40\% of training for FineWeb1B (left) and SlimPajama6B (right).}
\label{fig:fineweb1B_slimpajama1B_loss_curves}
\end{figure}

\begin{table}[t]
\centering
\caption{Validation loss for FineWeb1B with tuned LRs (mean $\pm$ std over three seeds).}
\begin{tabular}{lcccc}
\toprule
\alg{MuonAdam} & \alg{Scion} & \alg{MuonAdam}-\alg{Momo} & \alg{MuonMax}-\alg{Momo} \\
\midrule
$3.5592 \pm 0.0014$ & $3.5947 \pm 0.0031$ & $3.5546 \pm 0.0004$ & $3.5779 \pm 0.0007$ \\
\bottomrule
\end{tabular}
\label{tab:final_fineweb_losses} 
\end{table}

\begin{figure}
\centering
\includegraphics[width=0.49\textwidth]{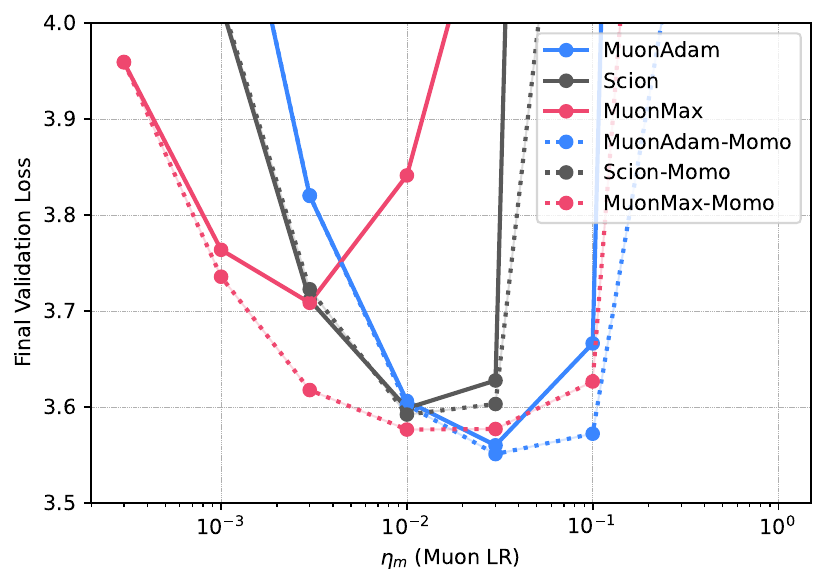}
\caption{Effect of model truncation on final validation loss. Note that for these runs,
we did not use stale nuclear norm approximations in order to isolate the effect of model
truncation.}
\label{fig:fineweb1B_trunc_comparison}
\end{figure}

\subsection{SlimPajama}
Figure \ref{fig:slim_pajama1B_heatmap} shows a 2D visualization of final validation
losses for \alg{Muon}, \alg{Scion}, \alg{Muon}-\alg{Momo}, and \alg{MuonMax}-\alg{Momo}
as the two learning rates vary. We find \alg{MuonMax}-\alg{Momo} to be most stable to
changes in the learning rates, with both \alg{Muon} and \alg{Scion} suffering high
losses when the base LR $\eta_b$ is large. Interestingly, \alg{Muon}-\alg{Momo} has the
highest loss when the Muon LR $\eta_m$ is small and the base LR $\eta_b$ is large.

\begin{figure}[t]
\centering
\includegraphics[width=0.99\textwidth]{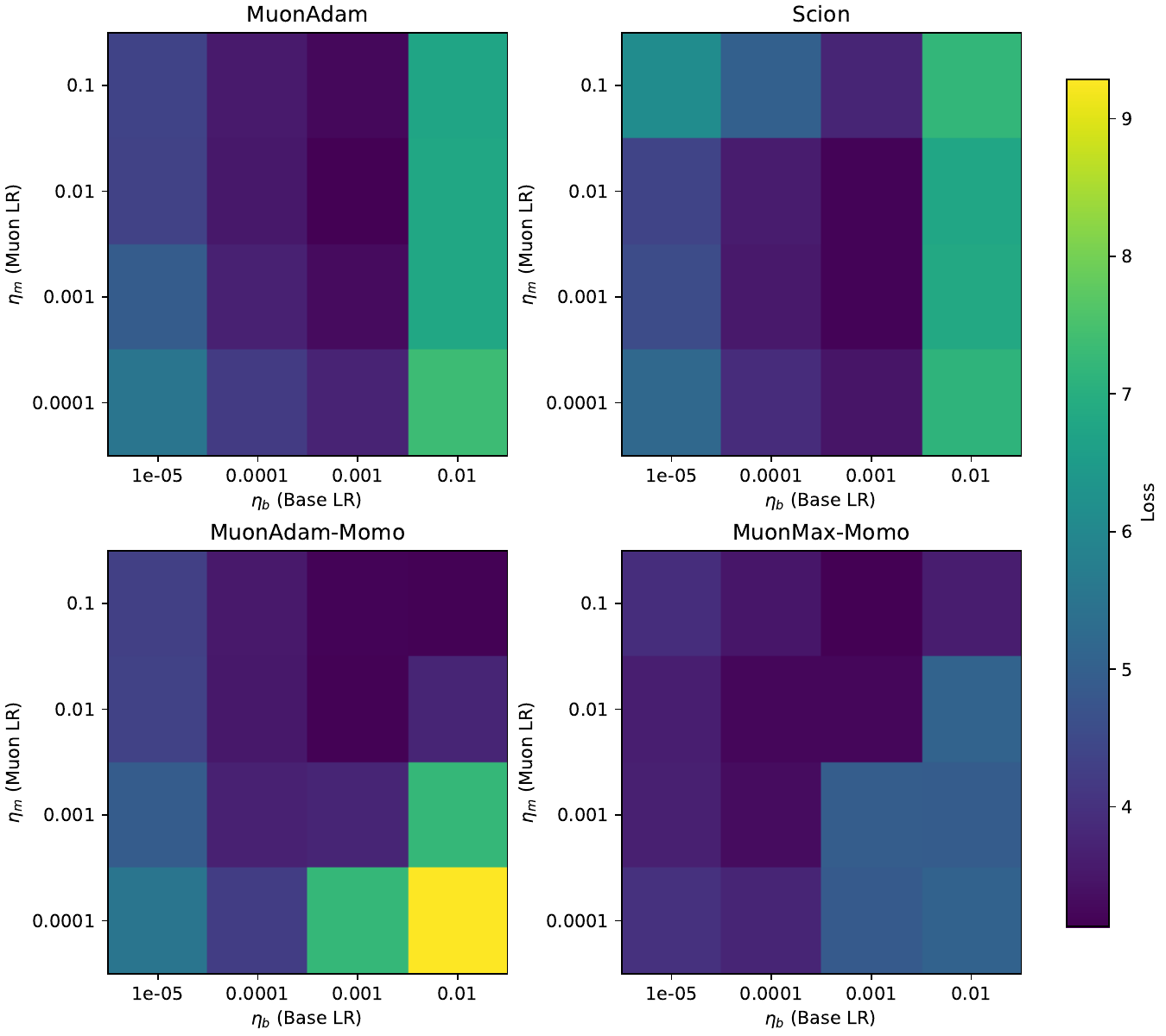}
\caption{2D learning rate sensitivity for SlimPajama1B.}
\label{fig:slim_pajama1B_heatmap}
\end{figure}

We also include loss curves for the last 40\% of training for \alg{MuonAdam} and
\alg{MuonMax}-\alg{Momo} (with tuned learning rates) in Figure
\ref{fig:slimpajama6B_loss_curves}.

\end{document}

%% file: math_commands.tex
\usepackage{amsmath,amsfonts,bm}

\def\eqref#1{equation~\ref{#1}}
\def\Eqref#1{Equation~\ref{#1}}

\def\1{\bm{1}}

\def\vzero{{\bm{0}}}
\def\vone{{\bm{1}}}

\def\vg{{\bm{g}}}

\def\vm{{\bm{m}}}

\def\vr{{\bm{r}}}
\def\vs{{\bm{s}}}

\def\vu{{\bm{u}}}
\def\vv{{\bm{v}}}
\def\vw{{\bm{w}}}

\def\vz{{\bm{z}}}

\def\mA{{\bm{A}}}
\def\mB{{\bm{B}}}

\def\mD{{\bm{D}}}

\def\mG{{\bm{G}}}

\def\mM{{\bm{M}}}

\def\mP{{\bm{P}}}

\def\mW{{\bm{W}}}

\DeclareMathAlphabet{\mathsfit}{\encodingdefault}{\sfdefault}{m}{sl}
\SetMathAlphabet{\mathsfit}{bold}{\encodingdefault}{\sfdefault}{bx}{n}

\newcommand{\R}{\mathbb{R}}

\DeclareMathOperator*{\argmax}{arg\,max}
\DeclareMathOperator*{\argmin}{arg\,min}

\DeclareMathOperator{\sign}{sign}